\documentclass[letterpaper,11pt]{article}

\usepackage[margin=1in]{geometry} 

\usepackage[utf8]{inputenc} 
\usepackage[T1]{fontenc}    
\usepackage{url}            
\usepackage{booktabs}       
\usepackage{amsfonts}       
\usepackage{nicefrac}       
\usepackage{microtype}      
\usepackage{amsmath,amssymb}
\usepackage{bm}
\usepackage{amsthm}         

\usepackage{xcolor}

\usepackage{amsmath,amsfonts,bm}
\usepackage{amssymb}

\def\eqref#1{equation~\ref{#1}}

\def\1{\bm{1}}

\def\mI{{\bm{I}}}

\DeclareMathAlphabet{\mathsfit}{\encodingdefault}{\sfdefault}{m}{sl}
\SetMathAlphabet{\mathsfit}{bold}{\encodingdefault}{\sfdefault}{bx}{n}

\newcommand{\Pxy}{P_{X,Y}}

\newcommand{\Pygx}{P_{Y|X}}
\newcommand{\Px}{P_{X}}
\newcommand{\Py}{P_{Y}}

\newcommand{\EE}[1]{\mathbb{E}\left[#1\right]}

\newcommand{\calX}{\mathcal{X}}

\newcommand{\calN}{\mathcal{N}}
\newcommand{\calY}{\mathcal{Y}}

\newcommand{\calF}{\mathcal{F}}
\newcommand{\calG}{\mathcal{G}}

\newcommand{\calL}{\mathcal{L}}

\newcommand{\bff}{\mathbf{f}}
\newcommand{\bg}{\mathbf{g}}
\newcommand{\bSigma}{\bm{\Sigma}}

\newcommand{\bA}{\mathbf{A}}
\newcommand{\bB}{\mathbf{B}}
\newcommand{\bC}{\mathbf{C}}

\newcommand{\bG}{\mathbf{G}}

\newcommand{\bF}{\mathbf{F}}
\newcommand{\bI}{\mathbf{I}}

\newcommand{\bU}{\mathbf{U}}
\newcommand{\bV}{\mathbf{V}}

\newcommand{\bL}{\mathbf{L}}

\newcommand{\bY}{\mathbf{Y}}
\newcommand{\bZ}{\mathbf{Z}}
\newcommand{\bx}{\mathbf{x}}
\newcommand{\by}{\mathbf{y}}

\newcommand{\bX}{\mathbf{X}}

\renewcommand{\tilde}{\widetilde}

\newcommand{\Reals}{\mathbb{R}}
\newcommand{\mtf}{\mathbf{\tilde{f}}}
\newcommand{\mtg}{\mathbf{\tilde{g}}}
\newcommand{\pxy}{p_{X,Y}}
\newcommand{\px}{p_X}
\newcommand{\py}{p_Y}
\newcommand{\pygx}{p_{Y|X}}

\usepackage{algorithmicx}
\usepackage{algorithm}
\usepackage[noend]{algpseudocode}
\algnewcommand\algorithmicinput{\textbf{Input:}}
\algnewcommand\Input{\item[\algorithmicinput]}
\algnewcommand\algorithmicoutput{\textbf{Output:}}
\algnewcommand\Output{\item[\algorithmicoutput]}

\usepackage{amsthm}
\newtheorem{prop}{Proposition}
\theoremstyle{remark}

\usepackage{graphicx}
\usepackage{epstopdf}
\usepackage{natbib}
\usepackage{subfig}
\usepackage[nottoc,numbib]{tocbibind} 
\title{Correspondence Analysis Using Neural Networks}
\date{} 
\author{
    Hsiang Hsu\thanks{Hsiang Hsu and Flavio P. Calmon are with John A. Paulson School of Engineering and Applied Sciences, Harvard University, Cambridge, MA (e-mails: \texttt{hsianghsu@g.harvard.edu}, \texttt{flavio@seas.harvard.edu}).}, 
    Salman Salamatian\thanks{Salman Salamatian is with the Research Laboratory of Electronics at the Massachusetts Institute of Technology, Cambridge, MA (e-mail: \texttt{salmansa@mit.edu}).},
    and Flavio P. Calmon${}^*$\\
}

\begin{document}

\maketitle
\begin{abstract}
  Correspondence analysis (CA) is a multivariate statistical tool used to visualize and interpret data dependencies. CA has found applications in fields ranging from epidemiology to social sciences.  
  However, current methods used to perform CA do not scale to large, high-dimensional datasets.
  By re-interpreting the objective in CA using an information-theoretic tool called the principal inertia components, we demonstrate that performing CA is equivalent to solving a functional optimization problem over the space of finite variance functions of two random variable. We show that this optimization problem, in turn, can be efficiently approximated by neural networks. The resulting formulation, called the correspondence analysis neural network (CA-NN), enables CA to be performed at an unprecedented scale.
  We validate the CA-NN on synthetic data,  and demonstrate how it can be used to perform CA on a variety of datasets, including  food recipes, wine compositions, and images. Our results outperform traditional methods used in CA, indicating that CA-NN can  serve as a new, scalable tool for interpretability and visualization of complex dependencies between random variables.
\end{abstract}
\textbf{Keywords}: Correspondence analysis, principal inertia components, principal functions, canonical correlation analysis.

\clearpage
\tableofcontents

\clearpage
\section{Introduction}
Correspondence Analysis (CA) is an exploratory multivariate statistical technique that converts data into a graphical display with orthogonal factors.
CA's history in the applied statistics literature dates back several decades \citep{benzecri1973correspondence, greenacre1984theory,lebart2013correspondence, greenacre2017correspondence}. 
In a similar vein to Principal Component Analysis (PCA) and its kernel variants \citep{hoffmann2007kernel}, CA is a technique that maps the data onto a low-dimensional representation. By construction, this new representation  captures possibly non-linear  relationships between the underlying variables, and can be used to interpret the dependence between two random variables $X$ and $Y$ from observed samples. CA has the  ability to produce interpretable,  low-dimensional visualizations (often two-dimensional) that capture complex relationships in data with entangled and intricate dependencies. This has led to its successful deployment in fields ranging from genealogy and epidemiology to social and environmental sciences \citep{tekaia2016genome, sourial2010correspondence, carrington2005models, ter2004co, ormoli2015diversity, ferrari2016whole}.

Despite being a versatile statistical technique, CA has been underused on the large datasets currently found in the machine learning landscape.
This can potentially be explained by the fact that, traditionally, CA utilizes as its main ingredient a singular value decomposition (SVD) of the normalized contingency table of $X$ and $Y$  (i.e., an empirical approximation of the joint distribution $\Pxy$).  
This contingency table-based approach for performing CA has three fundamental limitations.
First, it is restricted to data drawn from \emph{discrete} distributions with finite support, since  contingency tables for  continuous variables will be highly dependent on a chosen quantization which, in turn, may jeopardize information in the data.
Second, even when the underlying distribution of the data is discrete, reliably estimating the contingency table (i.e., approximating $P_{X,Y}$) may be infeasible due to limited number of samples. This inevitably hinges CA  on the more (statistically) challenging problem of estimating $P_{X,Y}$. 
Third, building contingency tables is not feasible for \emph{high-dimensional} data. For example, if $X \in \{0, 1\}^a$ and all outcomes have non-zero probability, then the contingency table has $2^a$ rows.

We address these limitations by taking a fresh theoretical look at CA and re-interpreting the low-dimensional representations produced by this technique from a functional analysis vantage point. We bring to bear an information-theoretic tool called the \emph{principal inertia components} (PICs) of a joint distribution $\Pxy$ \citep{du2017principal}. 
In essence, the PICs provide a fine-grained decomposition of the statistical dependency of $X$ and $Y$, fully determining an orthornormal set of finite-variance functions of $X$ that can be reliably estimated from $Y$ (and vice-versa) called the  \emph{principal functions} (PFs).  
The d\'ebute of PICs under different guises can be traced back to the works of  \citet{hirschfeld1935connection}, \citet{gebelein1941statistische}, and   \citet{renyi1959measures}. The  PICs are at the heart of the Alternating Conditional Expectations (ACE) algorithm \citep{buja1990remarks, breiman1985estimating} and have been studied in the information theory and statistics literature \citep{witsenhausen1975sequences, makur2015bounds, huang2017information, du2017principal}.

We demonstrate that the low-dimensional projections produced by CA are \emph{exactly} the principal functions found in the theory of PICs. The principal functions, in turn,  can be determined by solving a quadratic optimization problem over the space of finite variance functions of $X$ and $Y$. Solving this optimization for arbitrary variables is, at first glance, infeasible. However, by restricting our search to functions representable by multi-layer neural networks, we demonstrate how the principal functions can be efficiently approximated for both discrete and continuous (potentially high-dimensional) random variables. In summary,  by first formulating CA in terms of a PIC-based optimization program, and then approximating this program using neural networks, we are able to perform CA at an unprecedented scale. 

The contributions of this paper are as follows:
\begin{enumerate}
    \item We show how the PICs and principal functions can be used for correspondence analysis (Section~\ref{sec:background}).
    \item We introduce the Correspondence Analysis Neural Net (CA-NN) to estimate the PICs and principal functions,  thereby making CA scalable to discrete and continuous (high-dimensional) data (Section~\ref{sec:corrann}).
    \item We use synthetic data to demonstrate that the principal functions found by CA-NN match the functions predicted by theory (Section~\ref{sec:synthetic}). 
    \item Moreover, we apply the CA-NN on several real-world datasets, including images (MNIST \citep{lecun1998gradient}, CIFAR-10 \citep{krizhevsky2009learning}), recipes \citep{kaggle_what_cooking}, and UCI wine quality \citep{asuncion2007uci}. These examples demonstrate how the interpretable analysis and visualizations found in the CA literature can now be produced at a much larger scale (Section~\ref{sec:real_data}). All codes and experiments are available in \citep{Hsu2019}.
\end{enumerate}
\subsection{Related Work}\label{sec:related_works}
Several statistical methods exist for producing correlated low-dimensional representations of two variables $X$ and $Y$. For example, Canonical Correlation Analysis (CCA) \citep{hotelling1936relations}  seeks to find linear relationships between variables. Kernel Canonical Correlation Analysis (KCCA) \citep{bach2002kernel} extends this approach by first projecting the variables onto a reduced kernel Hilbert space. 
The method closest to the one described here is Deep Canonical Correlation Analysis (DCCA) \citep{andrew2013deep}, where  non-linear representations of multi-view data is produced using neural nets. The objective of DCCA in \citep[Eq. 1]{wang2015deep} is similar to finding the PICs. However, the non-linear projections found by DCCA are not exactly the principal functions, and  \citet{wang2015deep} do not make a connection to CA, PICs, nor Hilbert spaces.
DCCA is also closely related to maximal correlated PCA \citep{feizi2017maximally}.

PICs are a generalization of R\'enyi maximal correlation; in fact, the first PIC is identical to R\'enyi maximal correlation \citep{buja1990remarks}. Maximal correlation can be estimated using the Bivariate ACE algorithm, which determines non-linear projections $f(X)$ and $g(Y)$ that are maximally correlated whilst having zero mean and unit variance \citep{breiman1985estimating}. 
The projections are found by iteratively computing $\EE{g(Y)|X}$ and $\EE{f(X)|Y}$, and converge to the first principal functions.
However, for large, high-dimensional datasets, iteratively computing conditional expectations is intractable. To overcome this problem, neural-based approaches such as Correlational Neural Nets \citep{chandar2016correlational}  have been proposed. 
Unlike previous efforts, the approach that we outline here allows the PICs and principal functions to be \emph{simultaneously} computed, generalizing existing methods in the literature.

Finally, we mention  Karhunen-Lo\`eve Transform, Principal Component Analysis (PCA), and its kernel version (kPCA) \citep{hoffmann2007kernel}. Similar to CA, these methods aim at representing data in terms of orthogonal (uncorrelated) components. As such, they capture the structural relationship within a high dimensional random vector of features $X$. However, these methods disregard whether these representations are relevant from the point of view of another variable $Y$. Instead, CA finds orthogonal components of both $X$ and $Y$ \emph{jointly}, with the resulting components being highly correlated. Note that this is different from performing PCA or kPCA on the joint pair $(X,Y)$, as evidenced by the derivations in Section \ref{sec:background}. Moreover, unklike kPCA, CA produces non-linear, highly correlated representations without requiring a kernel to be defined \textit{a priori}.

\subsection{Notation}
Capital letters (e.g. $X$) are used to denote random variables, and calligraphic letters (e.g. $\calX$) denote sets.
We denote the probability measure of $X\times Y$ by $\Pxy$, the conditional probability measure of $Y$ given $X$ by $\Pygx$, and the marginal probability measure of $X$ and $Y$ by $\Px$ and $\Py$ respectively.
We denote the fact that $X$ is distributed according to $\Px$ by $X \sim \Px$.
If $X$ and $Y$ have finite support sets $|\calX| < \infty$ and $|\calY| < \infty$, then we denote the joint probability mass function (pmf) of $X$ and $Y$ as $\pxy$, the conditional pmf of $Y$ given $X$ as $\pygx$, and the marginal distributions of $X$ and $Y$ as $\px$ and $\py$, respectively. 
A sample drawn from a probability distribution is denoted by lower-case letters (e.g. $x$ and $y$).
Matrices are denoted in bold capital letters (e.g. $\mathbf{X}$) and vectors in bold lower-case letters (e.g. $\mathbf{x}$). The $(i,j)$-th entry of a matrix $\mathbf{X}$ is given by $[\mathbf{X}]_{i,j}$. We denote the identity matrix of dimension $d$ by $\mathbf{I}_d$, and the all-one vector of dimension $d$ by $\mathbf{1}_d$. Given $\mathbf{v} \in \Reals^d$, we denote the matrix with diagonal entries equal to $\mathbf{v}$ by $\mathsf{diag}(\mathbf{v})$.

\section{Correspondence Analysis and the Principal Inertia Components}\label{sec:background}
In this section, we formally introduce CA, the PICs and the principal functions, as well as the connection between the PICs and CA. 

\subsection{Correspondence Analysis}\label{sec:ca}
Correspondence analysis considers two random variables $X$ and $Y$ with $|\calX| < \infty$, $|\calY| < \infty$, and pmf $\pxy$ (cf. \citet{greenacre1984theory} for a detailed overview). Given samples $\{x_k, y_k\}_{k=1}^n$ drawn independently from $\pxy$, a two-way contingency table $\mathbf{P}_{X, Y}$ is defined as a matrix with $|\calX|$ rows and $|\calY|$ columns of normalized co-occurrence counts, i.e. $[\mathbf{P}_{X,Y}]_{i,j}=(\mbox{\# of observations } (x_i,y_i)=(i,j))/n$. 
Moreover, the marginals are defined as $\mathbf{p}_X \triangleq \mathbf{P}_{X, Y} \mathbf{1}_{|\mathcal{Y}|}$ and $\mathbf{p}_Y \triangleq \mathbf{P}_{X, Y}^T \mathbf{1}_{|\mathcal{X}|}$. 
Consider a matrix 
\begin{equation}
\label{eq:Q}
\mathbf{Q}\triangleq \mathbf{D}_{X}^{-1/2}(\mathbf{P}_{X,Y}-\mathbf{p}_X\mathbf{p}_Y^T)\mathbf{D}_{Y}^{-1/2},
\end{equation}
where $\mathbf{D}_{X} \triangleq \mathsf{diag}(\mathbf{p}_X)$ and $\mathbf{D}_{Y} \triangleq \mathsf{diag}(\mathbf{p}_Y)$, and let the SVD of $\mathbf{Q}$ be $\mathbf{Q} = \bU \bSigma \bV^\intercal$. Let $d = \min\{ |\calX|, |\calY| \}-1$, and $\{\sigma_i\}_{i=1}^d$ be the singular values, then we have the following definitions \citep{greenacre1984theory}:
\begin{itemize}
    \item The orthogonal factors of $X$ are $\mathbf{L} \triangleq \mathbf{D}_{X}^{-1/2} \bU$.
    \item The orthogonal factors of $Y$ are $\mathbf{R} \triangleq \mathbf{D}_{Y}^{-1/2} \bV$.
    \item The factor scores are $\lambda_i = \sigma_i^2, 1 \leq i \leq d$.
    \item The factor score ratios are $\frac{\lambda_i}{\sum_{i=1}\lambda_i}, 1 \leq i \leq d$.
\end{itemize}
CA makes use of the orthogonal factors $\mathbf{L}$ and $\mathbf{R}$  to visualize the correspondence (i.e., dependencies), between $X$ and $Y$. In particular, the first and second columns of $\mathbf{L}$ and $\mathbf{R}$  can be plotted on a two-dimensional plane (with each row corresponding to a point) producing the so-called \emph{factoring plane}. The remaining planes can be produced by plotting the other columns of $\mathbf{L}$ and $\mathbf{R}$. The factor score ratio quantifies the variance (``correspondence'') captured by each orthogonal factor, and is often shown along the axes in factoring planes. 

We provide next the definition of the PICs, which will enable the CA decomposition in (\ref{eq:Q}) to be performed for arbitrary random variables under appropriate compactness assumptions.

\subsection{Functional Spaces and the Principal Inertia Components}
For a random variable $X$ over the alphabet $\mathcal{X}$, we let $\calL_2(P_X)$ be the Hilbert Space of all functions from $\mathcal{X} \to \mathbb{R}$ with finite variance with respect to $P_X$, i.e., $\calL_2(\Px)\triangleq \left\{f:\calX\to \Reals \;\middle| \; \EE{\|f(X)\|_2}<\infty \right\}$. 
For $f_1, f_2 \in \mathcal{L}_2(P_X)$, this Hilbert space has an associated inner product given by $\langle f_1, f_2 \rangle = \EE{ f_1(X) f_2(X)}$.
As customary, this inner product induces a distance between two functions $f_1, f_2 \in \calL_2(P_X)$, namely the Mean-Square-Error (MSE) distance given by $d(f_1,f_2) = \mathbb{E}\left[ (f_1(X) - f_2(X))^2\right]$.
One can construct the projection operator from  $\calL_2(P_X)$ to  $\calL_2(P_Y)$ by 
\begin{eqnarray}\label{eq:mmse}
    \Pi_{Y=y} [f] &\triangleq& \underset{g \in \calL_2(P_Y)}{\mathrm{argmin}} \; \mathbb{E}_{X,Y}\left[ ( f(X) - g(Y))^2|Y=y\right]\nonumber\\
    &=& \mathbb{E}[f(X)|Y = y], 
\end{eqnarray}
with adjoint operator $\Pi_{X=x}[g]=\mathbb{E}[g(Y)|X = x]$  defined for $g \in \calL_2(P_Y)$. The projection operator describes the closest function, in terms of mean-square-error, to a given function $f$ of the inputs.
Since $\calL_2(P_X)$ is a Hilbert space, there exists a basis (in fact infinitely many) through which any function $f \in \calL_2(P_X)$ can be equivalently represented.
However, at a high level, it is of interest to find a basis for $\calL_2(P_X)$, which \emph{diagonalizes} the projection operator $\Pi_Y$.

This naturally leads to the following proposition.
\begin{prop}[\citet{witsenhausen1975sequences}]
\label{prop:defnPIC}
Without loss of generality, let $|\calY| \leq |\calX|$ and let $d \triangleq |\calY| -1$, or be infinity if both sets $\calX$ and $\calY$ are infinite. There exists two sets of functions of  $\mathcal{F} = \{ f_0, f_1, \ldots,f_{d}\}\subseteq \calL_2(\Px)$ and $\mathcal{G} = \{g_0, g_1, \ldots g_{d}\}\subseteq \calL_2(\Py)$, and a set $\mathcal{S} = \{1,\lambda_1,\ldots, \lambda_{d}\}$ such that:
\begin{itemize}
    \item $f_0(X)$ and $g_0(Y)$ are constant function almost surely,  $\mathbb{E}[f_i(X) f_j(X)] = \delta_{i,j}$ and $\mathbb{E}[g_i(Y) g_j(Y)] = \delta_{i,j}$ (orthornormality).
    \item $\mathbb{E}[f_i(X)|Y=y] = \sqrt{\lambda_i}g_i(y)$, and $\mathbb{E}[g_i(Y)|X=x] = \sqrt{\lambda_i} f_i(x)$ for all $i = 1, \ldots, d$ (diagonalization).
    \item Any function $g \in \calL_2(P_Y)$ can be represented as a linear combination $g(y) = \sum_{i = 0}^{d} \beta_i g_i(y)$. Similarly, any function $f \in \calL_2(P_X)$ can be represented as a linear combination $f(x) = f^{\perp}(x) + \sum_{i = 0}^{d} \alpha_i f_i(x)$, where $f^{\perp}$ is orthogonal to all $f_i$ for all $i = 0,1,\ldots,d$ (basis).
\end{itemize}
\end{prop}
The functions within the sets $\mathcal{F}$ and $\mathcal{G}$ are defined here as the \emph{principal functions} of $P_{X,Y}$, and the elements of $\mathcal{S}$ as the \emph{principal inertia components} of $P_{X,Y}$.
We call $f_i(x)$ and $g_i(y)$ the $i^\text{th}$ PFs of $X$ and $Y$, and $0 \leq \lambda_i \leq 1$ the $i^\text{th}$ PIC.
Without loss of generality, we let $\lambda_1 \geq \lambda_2 \geq \ldots \geq \lambda_{d}$. Moreover, $\sqrt{\lambda_1}$ is also known as R\'enyi maximal correlation. 
A more thorough introduction to PICs can be found in \citep{witsenhausen1975sequences, buja1990remarks,du2017principal}.

\subsection{The Reconstitution Formula and Correspondence Analysis}
As illustrated in (\ref{eq:mmse}), the principal functions precisely characterize the MSE-performance of estimating a function of $X$ from an observation $Y$ (and vice-versa). 
In fact, the PICs and principal functions can be used to reconstitute the joint distribution entirely \citep[Sec. 3]{buja1990remarks}, i.e. 
\begin{eqnarray}\label{eq:P_YgX}
\frac{\pxy(x, y)}{\px(x)\py(y)} = 1 + \sum_{i=1}^d \sqrt{\lambda_i} f_i(x) g_i(y).
\end{eqnarray}
This decomposition has also appeared in the CA literature \citep[Chap. 4]{greenacre1984theory}. This reconstitution formula is key for bridging the PICs and  CA, and enables us to generalize CA to continuous variables. We make this connections precise  in the following proposition, which demonstrates that the orthogonal factors found in CA are \emph{exactly} the principal functions.
\begin{prop}
\label{prop:CA_PIC}
If $|\calX|$ and $|\calY|$ are finite, we set $[\mathbf{F}]_{i, j} = f_j(i)$, $[\mathbf{G}]_{k, j} = g_j(k)$ for $1 \leq i \leq |\calX|$, $1 \leq j \leq d$ and $1 \leq k \leq |\calY|$, and let $\mathbf{\Lambda} = \textsf{diag}(\lambda_1, \cdots, \lambda_d)$.
Moreover, let $\mathbf{L}$, $\mathbf{R}$ and $\bSigma$ follow from Section~\ref{sec:ca} and assume the diagonal entries of $\bSigma$ are in descending order. 
Then, the PFs $\mathbf{F}$ and $\mathbf{G}$ are equivalent to the orthogonal factors $\mathbf{L}$ and $\mathbf{R}$ in the CA, and the factoring scores $\bSigma$ are the same as the PICs $\mathbf{\Lambda}$.
\end{prop}
\begin{proof}
See Appendix~\ref{app:proof_1}.
\end{proof}

\section{The Correspondence Analysis Neural Net (CA-NN)}\label{sec:corrann}
In the previous section, we demonstrated that the orthogonal factors found via CA are equivalent to the principal functions given by the PIC decomposition of $P_{X,Y}$ (Prop. \ref{prop:CA_PIC}). Thus, we can (at least in theory)  perform CA by computing principal functions directly, without having to build a contingency table first. Principal functions, in turn, are well-defined for both discrete and continuous (or mixed) $X$ and $Y$, enabling CA to be extended to a broader range of data types. For the rest of the paper, we use the term principal functions and PICs to indicate the orthogonal factors and factor scores, respectively.

Equations (\ref{eq:mmse}), (\ref{eq:P_YgX}), and  Prop. \ref{prop:defnPIC} suggest that the principal functions can be computed for arbitrary variables by finding maximally correlated functions in $\calL_2(P_X)$ and $\calL_2(P_Y)$. Finding such functions, however, require a search over the space of all finite-variance functions of $X$ or $Y$, which is not feasible for high dimensional data. Thus, in order to approximate the principal functions and scale up CA, we restrict our search to \emph{functions representable by neural nets}. Note that the output of any neuron of a feed-forward neural net that receives $X$ as an input can be viewed\footnote{We assume that the outputs of a neural network have finite variance --- a reasonable assumption since several gates used in practice have bounded value (e.g., sigmoid, tanh) and, at the very least, the output is limited by the number of bits used in floating point representations.} as a point in $\calL_2(P_X)$ (and equivalently when receiving $Y$ as input).

In this section, we introduce the Correspondence Analysis Neural Net (CA-NN). The CA-NN estimates the PICs and the principal functions of $\Pxy$ by minimizing an appropriately defined loss function (described next) using gradient descent and backpropagation. We will use the CA-NN to perform CA at scale. 

\begin{figure*}[t!]
\centering
\includegraphics[width=.9\textwidth]{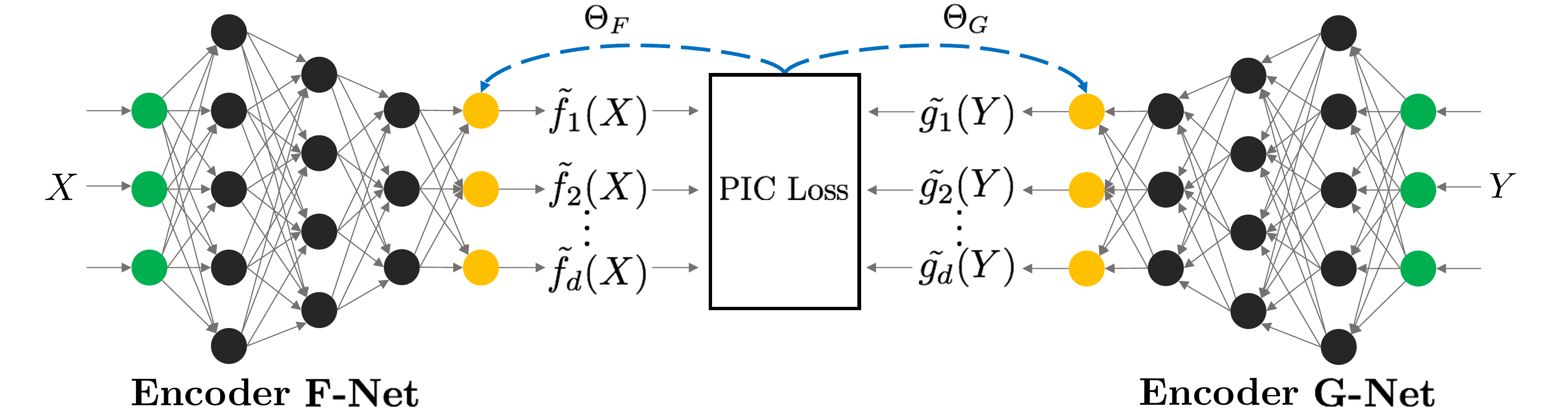}
\caption{{\small The architecture of the CA-NN, consisting of two encoders F-Net and G-Net for $X$ and $Y$ respectively to estimate the principal functions. The PIC loss is given by (\ref{opti4}).}}
\label{fig:fg_nets}
\end{figure*}

\subsection{Method}
For two random variables $(X,Y)$, we denote the $d$ principal functions of $X$ and $Y$, respectively, as
\begin{align}
    \mathbf{\tilde{f}}(X) &\triangleq [\tilde{f}_1(X), \cdots, \tilde{f}_d(X)]^\intercal \in \Reals^{d\times 1},\\
    \mathbf{\tilde{g}}(Y) &\triangleq [\tilde{g}_1(Y), \cdots, \tilde{g}_d(Y)]^\intercal \in \Reals^{d\times 1}.
\end{align}
Under these assumptions, the solution of the optimization problem
\begin{equation}\label{opti2}
\begin{aligned}
\min\limits_{\bA \in \Reals^{d\times d},\mathbf{\tilde{f}},\mathbf{\tilde{g}}} &\; \mathbb{E}\left[\|\bA\mathbf{\tilde{f}}(X)-\mathbf{\tilde{g}}(Y)\|^2_2\right]\\
\text{subject to}&\; \mathbb{E}\left[\bA\mathbf{\tilde{f}}(X)(\bA\mathbf{\tilde{f}}(X))^\intercal \right] = \mathbf{I}_d
\end{aligned}
\end{equation}
recovers the $d$ largest PICs.
To see why this is the case, let
\begin{equation}
    \bff(X)=\bA\mtf(X)=[\bff_1(X),\cdots,\bff_d(X)]^\intercal,
\end{equation}
and suppose that $\bff,\mtg$ and $\bA$ achieve optimality in (\ref{opti2}). Optimality under quadratic loss implies that $\tilde{g}_i(y)=\EE{f_i(X)\mid Y=y}$ for $i\in \{1,\dots,d\}$. Moreover, the orthogonality constraint assures that the entries of $\bff(X)$ satisfy $\EE{f_i(X)f_j(X)}=\delta_{i,j}$, and thus form a basis for a  $d$-dimensional subspace of $\calL_2(\Px)$. As discussed in Section~\ref{sec:background}, conditional expectation on $Y$ is a compact operator from $\calL_2(\Px)\to \calL_2(\Py)$, and from orthogonality of $\bff(X)$, it follows directly from the Hilbert-Schmidt Theorem \citep{reed1980functional} that the optimal value of (\ref{opti2}) is $\sum_{i=0}^{d-1}\lambda_i$, with $\bff$ and $\mtg$ corresponding to the $d$ largest principal functions.

We can further simplify the objective function in (\ref{opti2}) using the following proposition. 
\begin{prop}\label{prop:opti}
The minimization in (\ref{opti2}) is equivalent to the following unconstrained optimization problem. 
\begin{equation}\label{opti4}
\begin{aligned}
\min\limits_{\mathbf{\tilde{f}}, \mathbf{\tilde{g}}} && - 2\|\bC_f^{-\frac{1}{2}}\bC_{fg}\|_d + \mathbb{E}[\|\mathbf{\tilde{g}}(Y)\|^2_2],
\end{aligned}
\end{equation}
where $\bC_f = \mathbb{E}[ \mathbf{\tilde{f}}(X)\mathbf{\tilde{f}}(X)^\intercal ]$, $\bC_{fg} = \mathbb{E}[ \mathbf{\tilde{f}}(X)\mathbf{\tilde{g}}(Y)^\intercal ]$, and $\|\bZ\|_d$ is the $d$-th Ky-Fan norm, defined as the sum of the singular values of $\bZ$ \citep[Eq. (7.4.8.1)]{horn1990matrix}. Denoting by $\bA$ and $\bB$ the whitening matrices\footnote{We call $\bA$ and $\bB$ the whitening matrices since in Proposition~\ref{prop:defnPIC}, it is cleat that the covariance matrices of $\bff(X)$ and $\bg(Y)$ should both be identity matrices.}
for $\mathbf{\tilde{f}}(\bX)$ and $\mathbf{\tilde{g}}(\bY)$, the principal functions are given by $\bff(X) = [f_0(X), \cdots, f_d(X)]^\intercal = \bA\mtf(X)$ and $\bg(Y) = [g_0(Y), \cdots, g_d(Y)]^\intercal = \bB\mtg(Y)$.
\end{prop}
\begin{proof}
See Appendix~\ref{app:proof_2}.
\end{proof}
The proof of this propositon is related to the orthogonal Procrustes problem \citep{gower2004procrustes}, which convergence properties have been studied in \citep{nie2017generalized}.

\subsection{Implementation}\label{sec:implementation}
Observe that (\ref{opti4}) is an unconstrained optimization problem over the space of all finite variance functions of $X$ and $Y$. As discussed in the introduction of this section, we restrict our search to functions given by outputs of neural nets, parameterizing $\mathbf{\tilde{f}}(X)$ and $\mathbf{\tilde{g}}(Y)$ by  $\theta_F$ and $\theta_G$, respectively. Here, $\theta_F$ and $\theta_G$ denote weights of two neural nets, called the F-net and the G-net (Fig. \ref{fig:fg_nets}). The F-Net and the G-Net encode $X$ and $Y$ to $\Reals^d$, respectively. The parameters of each network can be fit using gradient-based back-propagation of the objective (\ref{opti4}), as described next.

Given $n$ samples $\{x_k, y_k\}_{k=1}^n$ from $\Pxy$, we denote $\bx_n \triangleq [x_1, \cdots, x_n]$, $\by_n \triangleq [y_1, \cdots, y_n]$, $\tilde{\bF}_n(\bx_n) = [\mathbf{\tilde{f}}(x_1, \theta_F), \cdots, \mathbf{\tilde{f}}(x_n, \theta_F)]^\intercal \in \Reals^{d\times n}$ and $\tilde{\bG}_n(\by_n) = [\mathbf{\tilde{g}}(y_1, \theta_G), \cdots, \mathbf{\tilde{g}}(y_n, \theta_G)] \in \Reals^{d\times n}$. The empirical evaluations of the terms in (\ref{opti4}) are
\begin{subequations}
\begin{eqnarray}
\bC_f &\approx& \frac{1}{n} \tilde{\bF}_n(\bx_n, \theta_F) \tilde{\bF}_n(\bx_n, \theta_F)^\intercal, \\
\bC_{fg} &\approx& \frac{1}{n} \tilde{\bF}_n(\bx_n, \theta_F) \tilde{\bG}_n(\by_n, \theta_G)^\intercal, \\
\mathbb{E}[\|\mathbf{\tilde{g}}(Y)\|^2_2] &\approx& \frac{1}{n} \sum_{i=1}^n\sum_{j=1}^d \mathbf{\tilde{g}}_j(y_i, \theta_G)^2.
\end{eqnarray}
\end{subequations}

After extracting $\tilde{\bF}_n(\bx_n)$ and $\tilde{\bG}_n(\by_n)$ from the F-Net and G-Net, respectively, Proposition~\ref{prop:opti} suggests that the principal functions can be recovered by producing whitening matrices $\bA$ for $\mtf$ and $\bB$ for $\mtg$.
Without loss of generality, we will assume that $\tilde{\bF}_n(\bx_n)$ and $\tilde{\bG}_n(\by_n)$ have zero-mean columns, which can always be done by subtracting the column-mean element-wise. Then $\bA$ is given by $\bA = \bU^\intercal \bC_f^{-1/2} $, with $\bU$ the left singular vectors of the matrix 
\begin{equation}
    \bL = \frac{1}{n} (\bC_f^{-1/2} \tilde{\bF}_n(\bx_n))(\bC_g^{-1/2} \tilde{\bG}_n(\by_n))^\intercal.
\end{equation}
The matrix $\bC_f^{-1/2}$ guarantees that $\tilde{\bF}_n(\bx_n)$ has orthonormal columns, while $\bU$ rotates the set of vectors to align with $\tilde{\bG}_n(\by_n)$.
By symmetry, $\bB = \bV^\intercal \bC_g^{-1/2}$, where $\bV$ are the right singular vectors of $\bL$. The produced matrices $\bF_n(\bx_n) = \bA \tilde{\bF}_n(\bx_n)$ and $\bG_n(\by_n) = \bB \tilde{\bG}_n(\by_n)$ satisfy
\begin{equation}
    \frac{1}{n} \bF_n(\bx_n)^\intercal\bF_n(\bx_n) = \frac{1}{n} \bG_n(\by_n)^\intercal \bG_n(\by_n) = \bI_d,
\end{equation}
and $\frac{1}{n}\bF_n(\bx_n)^\intercal\bG_n(\by_n) = \mathbf{\Lambda}$ being the diagonal matrix with the estimated PICs. It should be emphasized that in the implementation and subsequent experiments, we estimate the whitening matrices $\bA$ and $\bB$ on the training set alone prior to evaluation on the test set. 
For clarity, we summarize this whitening process in Appendix~\ref{app:algo}.

\section{Experiments}\label{sec:exp}
The experiments contain two parts. First, we apply the CA-NN to synthetic data where the PICs and principal functions can be computed analytically, and demonstrate that the proposed method recovers the values predicted by theory.
Second, we use the CA-NN to perform CA on two real-world datasets where contingency table-based CA fails: the Kaggle What's Cooking Recipes \citep{kaggle_what_cooking} and UCI Wine Quality data \citep{asuncion2007uci}.  In particular, the UCI Wine Quality dataset includes a mixture of discrete and continuous variables, demonstrating the versatility of the proposed methods. We select these datasets since their features naturally lend themselves to interpretable visualizations of the underlying dependencies in the data by factoring planes. In order to demonstrate that the CA-NN can be used to perform CA at an unprecedented scale, we also apply this method to image datasets (MNIST \citep{lecun1998gradient} and CIFAR-10 \citep{krizhevsky2009learning}), albeit these experiments do not lend themselves to the same kind of interpretable analysis found in the food related datasets. Detailed experimental setups (e.g., architecture of the CA-NN, training details, depths of encoders etc.) are provided in the Appendix~\ref{app:detail}, and an additional experiment on multi-modal Gaussian is provided in the Appendix~\ref{app:add_exp}.
All the $95\%$-confidence intervals of the estimation of the PICs in Tables~\ref{tab:pic_synthetic} and \ref{tab:pic_kaggle} are less than $1\%$ for CA-NN. 
\begin{figure}[!tb]
\centering
\includegraphics[width=.9\textwidth]{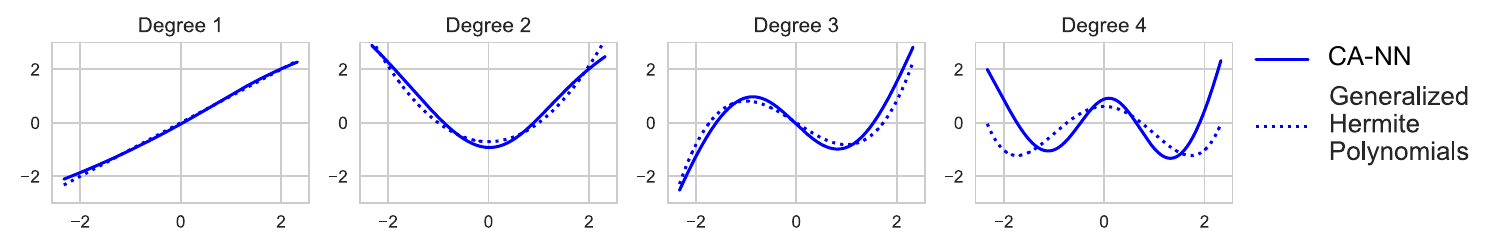}
\caption{{\small CA-NN recovers the Hermite polynomials, the principal functions in the Gaussian example.}}
\label{fig:hermite}
\end{figure}
\begin{table}[t]
\caption{{\small CA-NN reliably approximates the top four PICs in discrete and Gaussian cases.}}
\label{tab:pic_synthetic}
\begin{center}
\begin{tabular}{llllllllll}
& \multicolumn{4}{c}{\bf Discrete PICs} & \multicolumn{4}{c}{\bf Gaussian PICs}
\\ \hline \\
CA-NN      & $0.8011$ & $0.7942$ & $0.7918$ & $0.7883$ & $0.7007$ & $0.4938$ & $0.3376$ & $0.2037$ \\
Analytic value & $0.8000$ & $0.8000$ & $0.8000$ & $0.8000$ & $0.6977$ & $0.4675$ & $0.2979$ & $0.2113$ \\
\end{tabular}
\end{center}
\end{table}

\subsection{Synthetic Data}\label{sec:synthetic}
We demonstrate next through two examples --- one on discrete data and one on continuous data --- that the CA-NN is able to reliably recover the PICs and the principal functions predicted by theory.

\subsubsection{Discrete Synthetic Data}
We consider $X \sim Bernoulli(p)$, $Z \sim Bernoulli(\delta)$, and $Y = X \oplus Z$, where $\oplus$ is the exclusive-or operator and $\delta$ is the crossover probability. Note that $Y$ can be viewed as the output of a  discrete memoryless binary symmetric channel (BSC) \citep{cover2012elements} with input $X$. For any additive noise binary channel, the  PICs can be mathematically determined \citep{o2014analysis, du2017principal}. We set $X$ to be a binary string of length $5$, $\delta = 0.1$, and $p = 0.1$. The results in Table~\ref{tab:pic_synthetic} show that the CA-NN reliably approximates the PICs, and we observed this consistent behaviour over multipe runs of the experiment. Details (including analytical expressions for the PICs and principal functions)  are given in the Appendix.

\subsubsection{Gaussian Synthetic Data}
When $X \sim \calN(0, \sigma_1^2\mI_n)$, $Y|X \sim \calN(X, \sigma_2^2\mI_n)$, the set of principal functions are the Hermite polynomials \citep{abbe2012coordinate}. More precisely, if the $i^\text{th}$ degree Hermite polynomial is given by 
\begin{equation}
    H_i^{(r)}(x) \triangleq \frac{(-1)^i}{\sqrt{i!}} e^{\frac{x^2}{2r}} \frac{d^i}{dx^i} e^{-\frac{x^2}{2r}}, r \in (0, \infty),
\end{equation}
 then the $i^\text{th}$ principal functions $f_i$ and $g_i$ are $H_i^{(\sigma_1)}$ and $H_i^{(\sigma_1+\sigma_2)}$ respectively, and the $i^\text{th}$ PIC can then be given by their inner product. 
We pick $\sigma_1 = \sigma_2 = 1$, and show estimation of PICs and principal functions in Table~\ref{tab:pic_synthetic} and Fig.~\ref{fig:hermite}. Observe that the CA-NN closely approximates the first Hermite polynomials.

\subsection{Real-World Data}\label{sec:real_data}
We first investigate two datasets, Kaggle What's Cooking Recipes\citep{kaggle_what_cooking} and UCI Wine Quality \citep{asuncion2007uci}. These dataset contain highly non-linear dependencies which are interpretable via CA. In order to illustrate that we can perform CA on high-dimensional, continuous data, we conclude this section by applying the CA-NN to two image datasets.

\subsubsection{Kaggle What's Cooking Recipe Data}
\begin{table*}[!tb]
\caption{{\small CA-NN outperforms contingency table-based CA (SVD) and CCA/ KCCA (which also produces transformation of $X$ and $Y$, cf. Section~\ref{sec:related_works}) on Kaggle What's Cooking dataset to explore dependencies in samples.}}
\label{tab:pic_kaggle}
\begin{center}
\begin{tabular}{lclclclclclclclclclcl}
\multicolumn{1}{c}{\bf }  &\multicolumn{10}{c}{\bf Top ten principal inertia components} 
\\ \hline \\
CA-NN &  \boldsymbol{$0.9092$} & \boldsymbol{$0.8667$} & \boldsymbol{$0.8412$} & \boldsymbol{$0.7932$} & \boldsymbol{$0.7391$} & \boldsymbol{$0.6413$} & \boldsymbol{$0.6018$} & \boldsymbol{$0.4792$} & \boldsymbol{$0.4508$} & \boldsymbol{$0.2821$} \\
SVD  & $0.4504$ & $0.3894$ & $0.3149$ & $0.2943$ & $0.2413$ & $0.1958$ & $0.1547$ & $0.1191$ & $0.1146$ & $0.1035$ \\
\multicolumn{1}{c}{\bf }  &\multicolumn{10}{c}{\bf Correlations between transformed samples} 
\\ \hline \\
CCA  & $0.1915$ & $0.1751$ & $0.1342$ & $0.1083$ & $0.1050$ & $0.0823$ & $0.0623$ & $0.0488$ & $0.0485$ & $0.0431$ \\
KCCA & $0.6585$ & $0.1223$ & $0.0860$ & $0.0636$ & $0.0320$ & $0.0131$ & $0.0090$ & $0.0089$ & $0.0051$ & $0.0011$ \\
\end{tabular}
\end{center}
\end{table*}
\begin{figure*}[!tb]
\centering
\includegraphics[width=.9\textwidth]{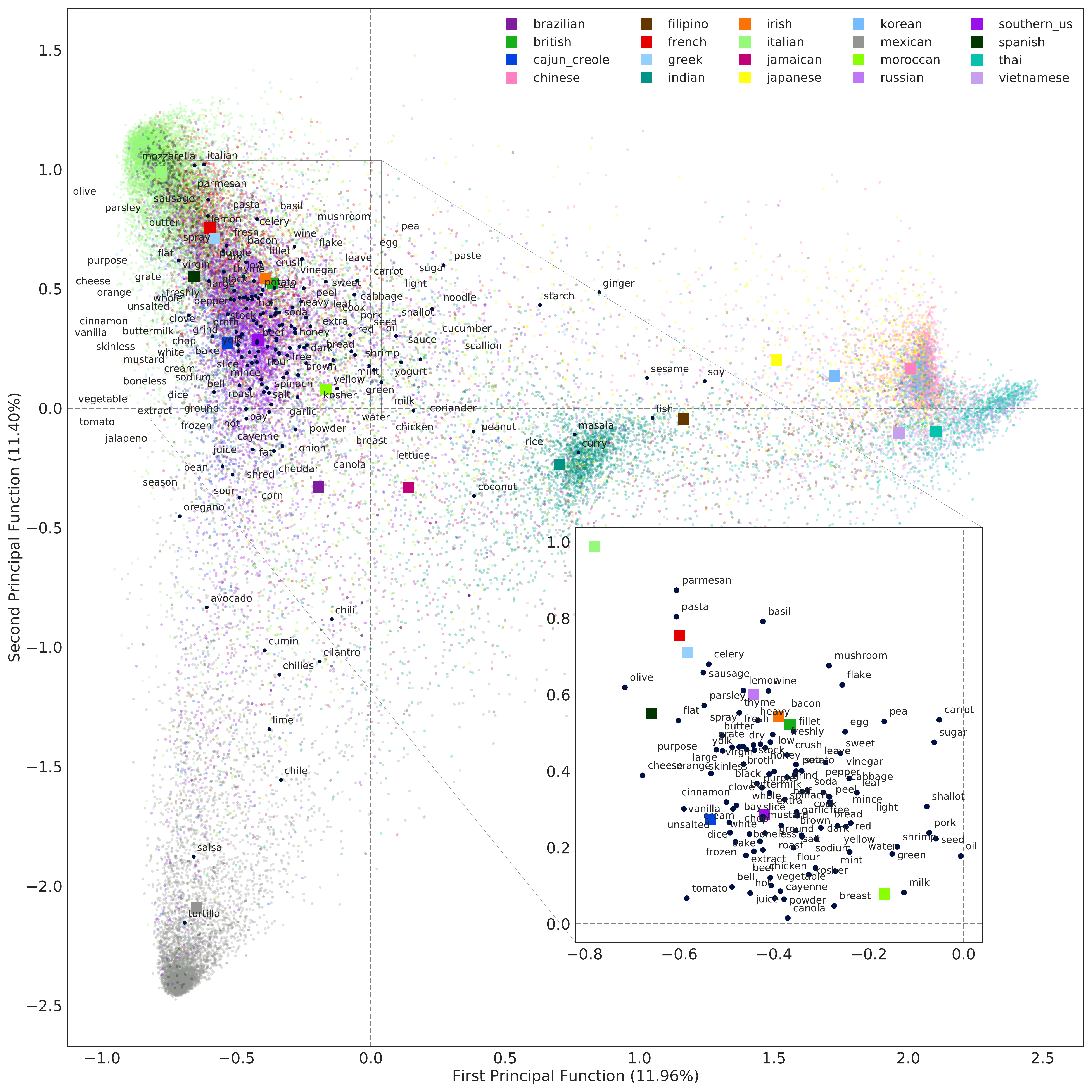}
\caption{{\small The first factoring plane of CA on Kaggle What's cooking dataset (Colored dots: recipe, dark blue: ingredient).}}
\label{fig:kaggle_ca}
\end{figure*}

\begin{figure}[!tb]
\centering
\includegraphics[width=.9\textwidth]{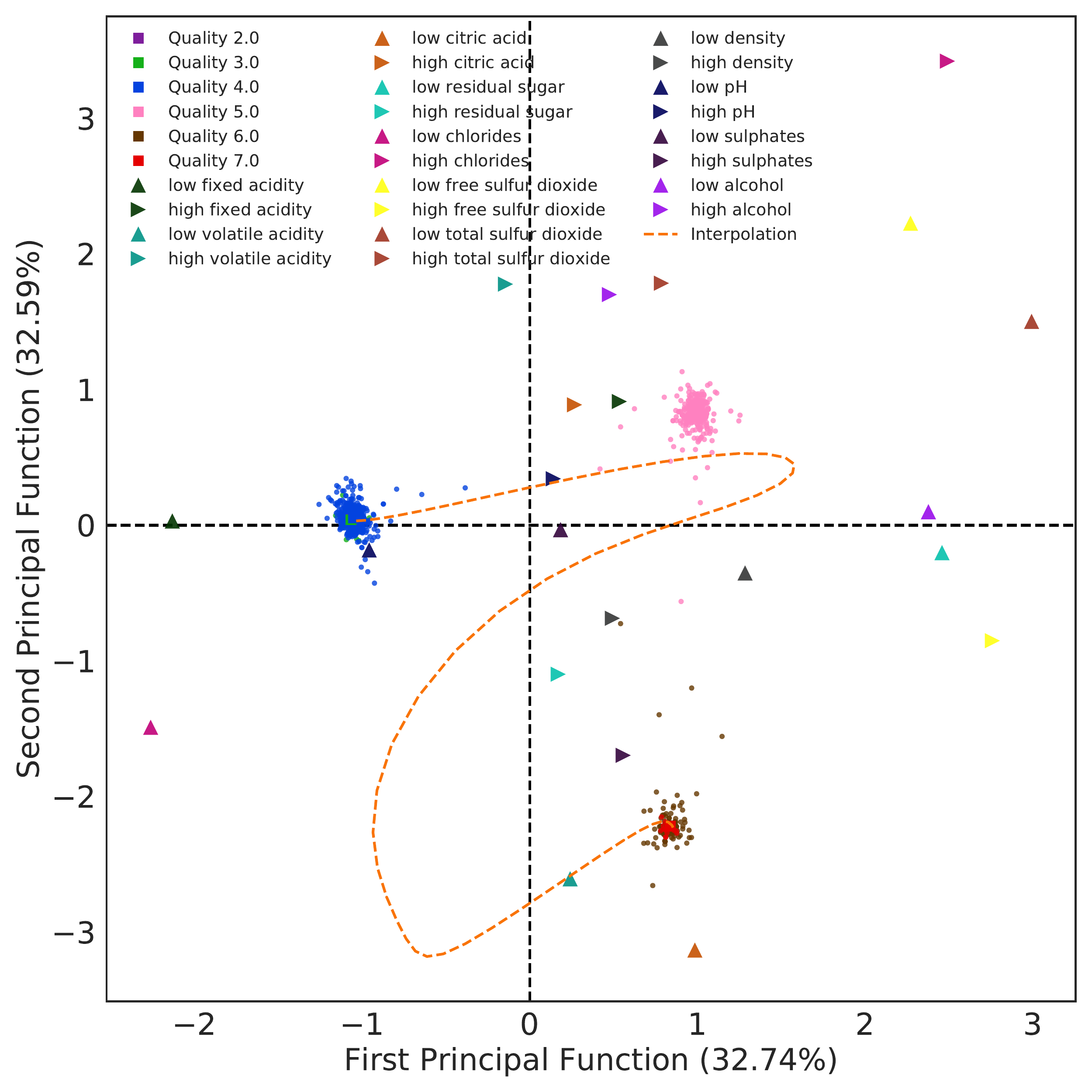}
\caption{{\small The first factoring plane of CA on UCI wine quality dataset.}}
\label{fig:wine_ca}
\end{figure}
The Kaggle What's cooking dataset \citep{kaggle_what_cooking} contains $39774$ recipes as $X$, composed of $6714$ ingredients (e.g. peanuts, sesame, beef, etc.), from $20$ types of cuisines as $Y$ (e.g. Japanese, Greek, Southern US, etc.). 
The recipes are given in text form, so we pre-process the data in order to combine variations of the same ingredient and, for the sake of example and interpretable visualizations, keep only the $146$ most common ingredients. 
In Table~\ref{tab:pic_kaggle}, we show that CA-NN outperforms contingency table-based CA using SVD\footnote{We only consider combinations of ingredients observed in the dataset as possible outcomes of $X$.}, with the resulting PICs being more correlated (i.e., achieving a higher value of Eq. (\ref{opti4})) than its contingency table-based counterpart. 

In Table~\ref{tab:pic_kaggle}, we compare the CA-NN with baseline techniques such as CCA and KCCA (with radial basis function kernels).
The resulting low-dimensional representation produced by CA-NN captures a higher correlation/variance than these other embeddings. We recognize that these results may vary if other kernels were selected, but note that the CA-NN \emph{does not require any form of kernel selection by a human prior to application}. 
In Fig.~\ref{fig:kaggle_ca} we display a traditional CA-style plot produced using CA-NN, showing the first factoring plane of the CA (i.e., the first and second principal functions for $X$ and $Y$). 
The intersection of two dashed lines ($x=0$ and $y=0$) indicates the space where $X$ and $Y$ have insignificant correlation.

There are three key observations which can be extracted from Fig.~\ref{fig:kaggle_ca}.
First, we observe clear clusters under the representation learned by CA-NN which can be easily interpreted. The cluster on the right-hand side represents East-Asian cuisines (e.g. Chinese, Korean), the one on the left-hand side represents Western cuisine (e.g. French, Italian) and in between sits Indian cuisine. 
Second, we observe that the first principal function learns to distinguish Asian cuisine (e.g. Chinese, Korean) from Western cuisine, naturally separating these contrasting culinary cultures. Interestingly, Indian cuisine sits in between Asian and Western cuisine, and  Filipino cuisine  is represented between Indian and Asian cuisine over this axis. The second principal function further indicates finer differences between Western cuisines, and singles out Mexican cuisine. 
Third, by plotting the ingredients on this plane (i.e., recipes containing only one ingredient), we can determine \emph{signature} ingredients for different kinds of cuisines. For example, despite the fact that ginger is in both Western and Asian dishes, it is closer in the factor plane to Asian cuisine, revealing that it plays a more prominent role in this cuisine. Some ingredients share much stronger correlation with the cuisine type, e.g. curry in Indian dishes and tortilla in Mexican ones.
See the Appendix for additional factorial planes.

\subsubsection{UCI Wine Quality Data}
The UCI wine quality dataset contains $4898$ red wines with $11$ physico-chemical attributes (e.g. pH value, acid, alcohol) and $6$ levels of qualities (from 2 to 7). 
We set $X$ to be the $11$ attributes and $Y$ be the quality, and report the results of CA in Fig.~\ref{fig:wine_ca}. 
Note that since the attributes are continuous, performing contingency table-based CA is not well-defined for this case.
In Fig.~\ref{fig:wine_ca}, we can see that despite the existence of $6$ classes of qualities, the principal functions discover three sub-clusters, namely poor quality (less or equal to 4), medium quality (equal to 5), and high quality (6 and above)
Moreover, Fig.~\ref{fig:wine_ca} also shows how the attributes affects the quality of a wine. For example, high quality wines (quality $6$ and $7$) tends to have low citric and volatile acidity, but with high sulphates. 
Finally, we randomly sample a low quality and a high quality wine, and take the linear interpolation of their features. We represent the path that this linear interpolation draws in the factorial plane by the orange line in Fig.~\ref{fig:wine_ca}. This sheds light on how a ``bad'' wine can be transformed into a ``good'' wine.

\subsubsection{MNIST and CIFAR-10}
In Fig.~\ref{fig:mnist_cifar10_ca}, we show that CA-NN enables CA on image datasets such as noisy MNIST \citep{wang2015deep} and CIFAR-10 \citep{krizhevsky2009learning} datasets.
Noisy MNIST is a more challenging version of MNIST \citep{lecun1998gradient}, containing $60000$ gray-scale images for training and $10000$ for test, where each image is a $28\times 28$ pixels handwritten digit with random rotation and noise; CIFAR-10 contains $50000$ colored images for training and $10000$ for test in $10$ classes, where each images has $32\times 32$ pixels.
Since the features (pixels) of these images are not very informative, we do not show them in Fig.~\ref{fig:mnist_cifar10_ca}.

\begin{figure}[t!]
\centering
\includegraphics[width=.8\textwidth]{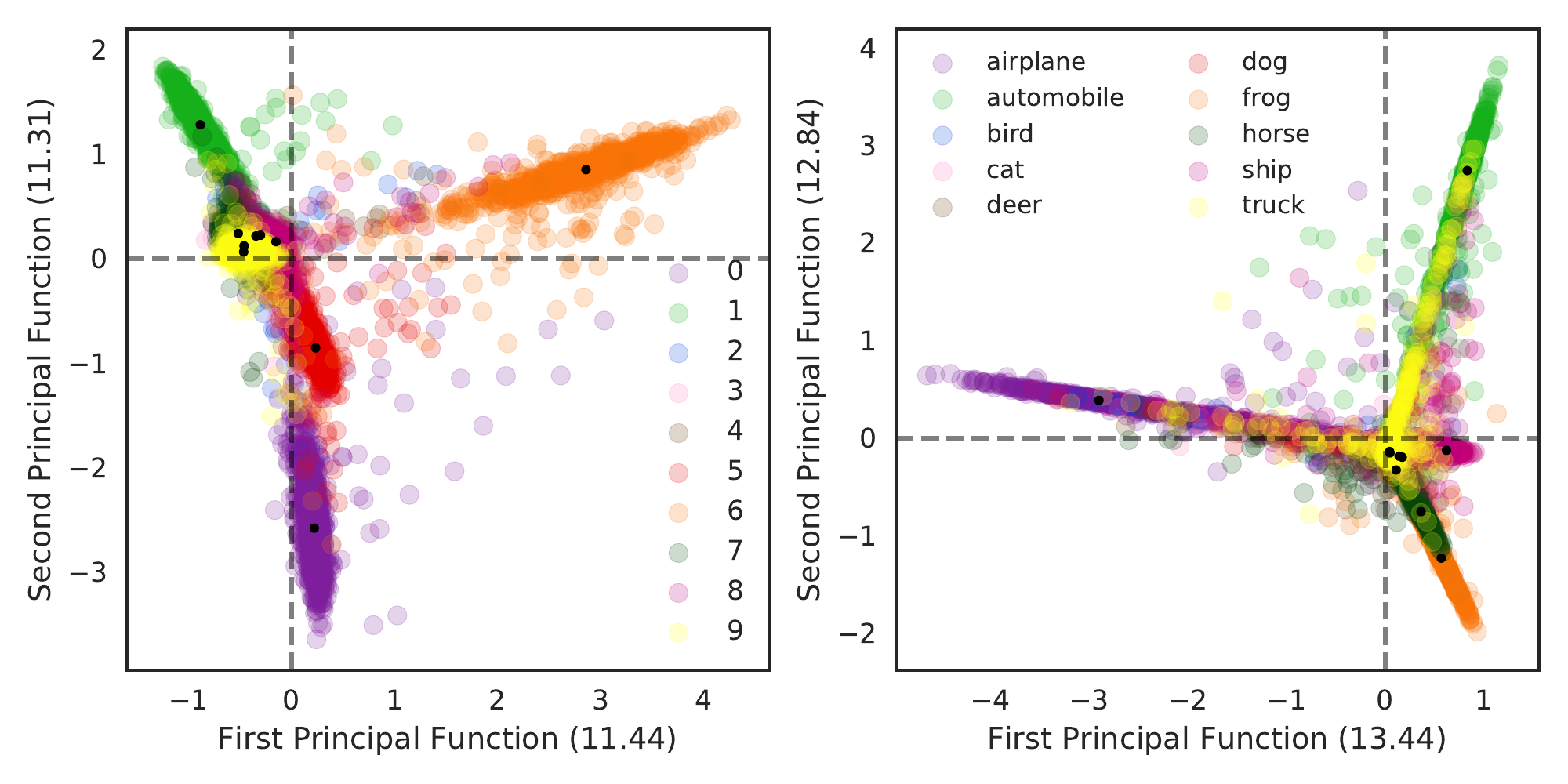}
\caption{{\small The first factoring planes of noisy MNIST (left) and CIFAR-10 (right).}}
\label{fig:mnist_cifar10_ca}
\end{figure}

\section{Conclusion}\label{sec:discussion}
We proposed a neural-based estimator for the principal inertia components and the principal functions, called the Correspondence Analysis Neural Network (CA-NN). By proving that the principal functions are equivalent to orthogonal factors in CA, we are able to use the CA-NN to scale up CA to large, high-dimensional datasets with continuous features. We validated the CA-NN on synthetic data, and showed how it enables CA to be performed on large real-world datasets. These experiments indicate that CA-NN significantly outperforms other approaches of CA. Future research directions include characterizing generalization properties of CA-NN in terms of the number of training samples, as well as the impact of network architecture on the resulting principal functions. We hope that the CA-NN can allow CA to be more widely applied to the large datasets found in the current machine learning landscape.

\bibliographystyle{apalike}
\bibliography{aistats2019.bib}

\clearpage
\appendix
\section{Experimental Details}\label{app:detail}
\subsection{Discrete Synthetic Data: Binary Symmetric Channels}
Explicit calculation of PICs between two given random variables is challenging in general; however, for some simple cases, e.g. $\Pygx$ given by a so-called discrete memoryless Binary Symmetric Channel (BSC), the PICs can be derived exactly \citep[Section~3.5]{du2017principal} or \citep[Section~2.4]{o2014analysis}. 
Let $X$ be a binary string of length $n$, and consider a binary string $Y$ of the same length, where each bit is flipped independently with probability $\delta$. 
The parameter $\delta$, called the \emph{crossover probability}, captures how noisy the mapping from $X$ to $Y$ is. 
By symmetry it is sufficient to let $\delta \leq 1/2$. The PICs between $X$ and $Y$ are characterized below: there are $\binom{n}{k}$ PICs of value $(1-2\delta)^k$. For example, for $n = 5$ and $\delta = 0.1$, there are $\binom{5}{0} = 1$ PIC of value $(1-0.2)^0 = 1$, $\binom{5}{1} = 5$ PICs of value $(1-0.2)^1 = 0.8$, $\binom{5}{2} = 10$ PICs of value $(1-0.2)^2 = 0.64$, and so on.  

For this experiment, we randomly generate $15000$ binary strings for training and $1500$ strings for testing.
The CA-NN is composed of simple neural nets with two hidden layers with ReLU activation, and $32$ units per hidden layer. We train over the entire training set for $2000$ epochs using a gradient descent optimizer with learning rate $0.01$.
The approximated PICs for the training and test set, along with the PICs values obtained analytically from theory are shown in Figure~\ref{fig:bsc_pics}.  The approximated PICs are close to the theoretical values, verifying that the CA-NN is valid in this example. 

We also show the factoring planes under different crossover probability $\delta$ in Figure~\ref{fig:bsc_delta}. When $\delta = 0.1$, most bits are identical between $X$ and $Y$, while when $\delta = 0.9$ most of the bits are flipped.

\subsection{Gaussian Synthetic Data and Hermite Polynomials}
When $\calX = \calY = \Reals$, $X \sim \calN(0, \sigma_1)$, $Z \sim \calN(0, \sigma_2)$ and $Y = X + Z$, the set of functions $\calF$ and $\calG$ that give the PICs are the Hermite polynomials \citep{abbe2012coordinate}, where for $x\in \Reals$, the Hermite polynomial $H_i(x)$ of degree $i \geq 0$ is defined as 
\begin{equation}
    H_i(x) \triangleq (-1)^i e^{\frac{x^2}{2}} \frac{d^i}{dx^i} e^{-\frac{x^2}{2}}.
\end{equation}
More precisely, the $i^\text{th}$ principal functions $f_i$ and $g_i$ are $H_i^{(\sigma_1)}$ and $H_i^{(\sigma_1+\sigma_2)}$ respectively, where $H_i^{(r)}$ denotes the generalized Hermite polynomial, defined as $H_i^{(r)}(x) = \frac{1}{\sqrt{i!}} H_i(\frac{x}{\sqrt{r}})$,
of degree $i$ with respect to the Gaussian distribution $\calN(0, r)$, for $r \in (0, \infty)$.
The PICs will then be given by the associated inner product $\mathbb{E}[H_i^{(\sigma_1)}(X)H_i^{(\sigma_1+\sigma_2)}(Y)]$.

We pick $\sigma_1 = \sigma_2 = 1$, and generate $5000$ training samples for $X$ and $Y$ according to the Gaussian distribution and $1000$ test samples.
The CA-NN is composed of two hidden layers with hyperbolic tangent activation, $30$ units per hidden layer.
We train over the entire training set for $8000$ epochs using a gradient descent optimizer with learning rate $0.01$.

In Figure~\ref{fig:hermite_4}, we show the Hermite polynomials of degrees $0$ to $4$ and the outputs of the CA-NN that approximate the $0^\text{th}$ to $4^\text{th}$ principal functions.
The output of the CA-NN closely recovers the Hermite polynomials; this can be further verified by computing the mean square difference between the approximated principal functions and the Hermite polynomials, i.e.
\begin{eqnarray}
\mathsf{MSE}_f &\triangleq& \mathbb{E} [(f_i(X) - H_i^{(\sigma_1)}(X))^2],\\
\mathsf{MSE}_g &\triangleq& \mathbb{E} [(g_i(Y) - H_i^{(\sigma_1+\sigma_2)}(Y))^2].
\end{eqnarray}
Table~\ref{tab:hermite} provides the mean square difference, as well as the theoretical and estimated PICs.
Since the CA-NN approximates the Hermite polynomials, the estimated PICs are also close to their theoretical values.

\begin{figure}
\centering
\includegraphics[width=.8\textwidth]{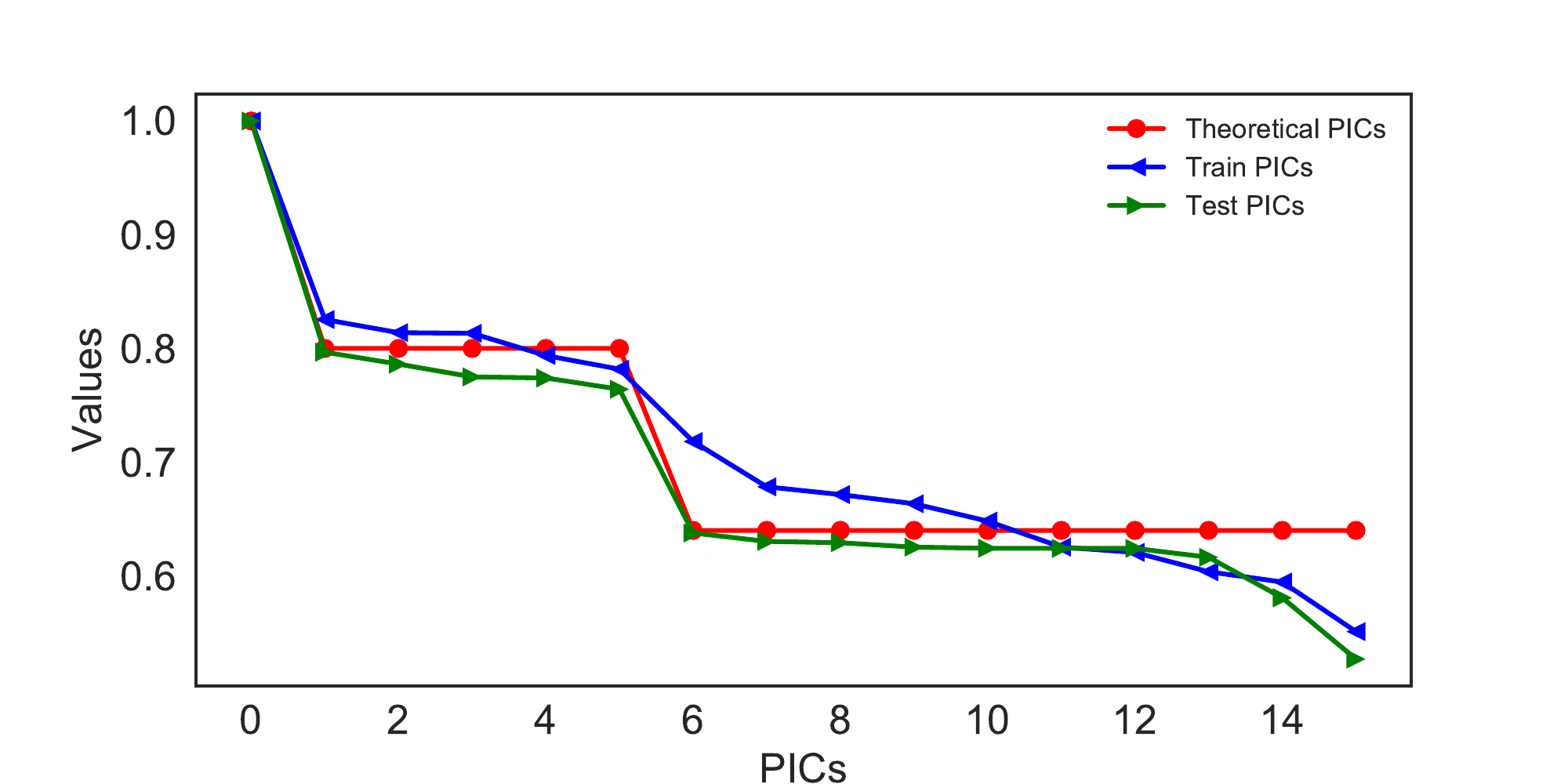}
\caption{Theoretical and approximated PICs between inputs and outputs of a BSC.}
\label{fig:bsc_pics}
\end{figure}

\begin{figure}
    \centering
    \subfloat[Crossover probability $\delta = 0.1$]{
        \includegraphics[width=0.5\textwidth]{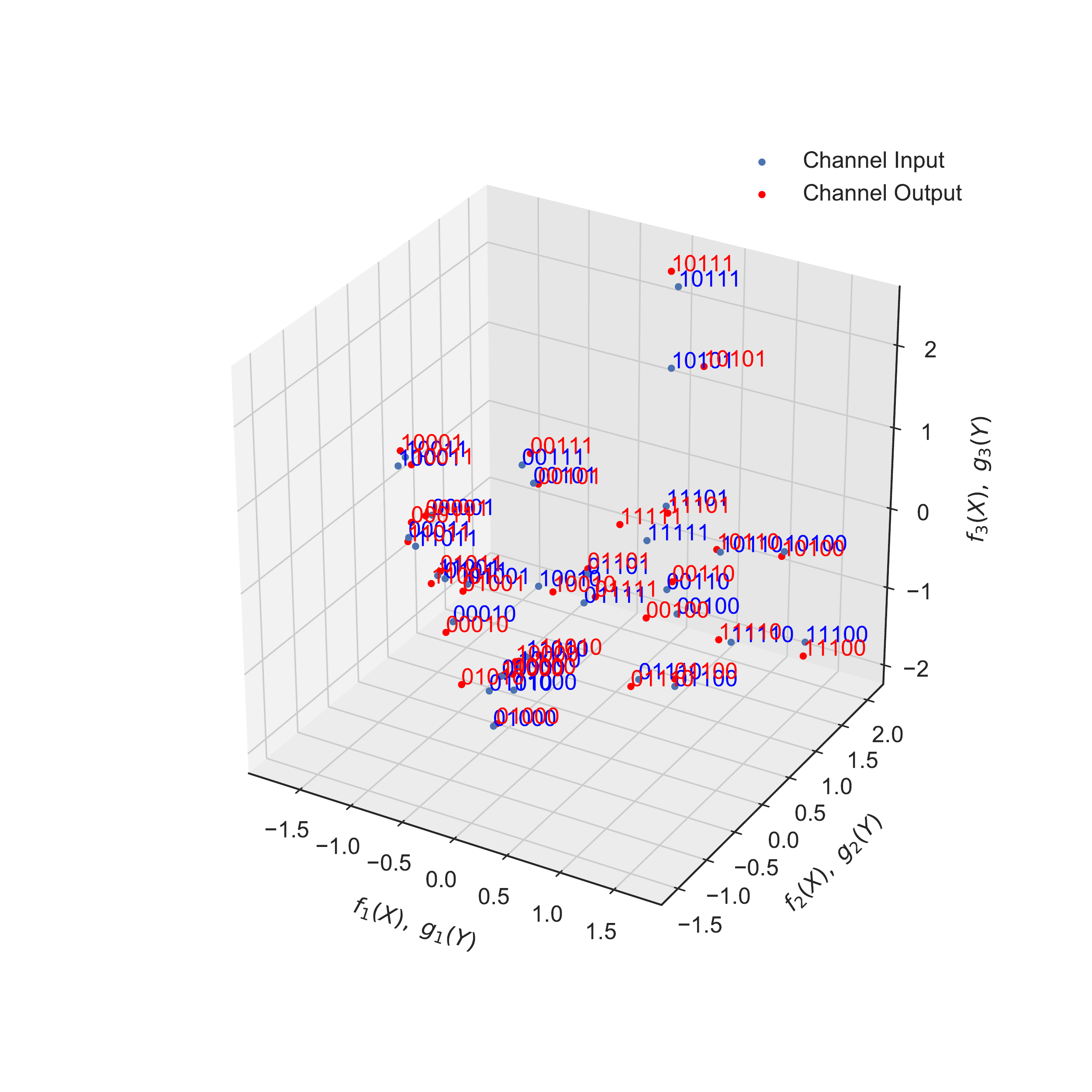}
    }
    \subfloat[Crossover probability $\delta = 0.9$]{
        \includegraphics[width=0.5\textwidth]{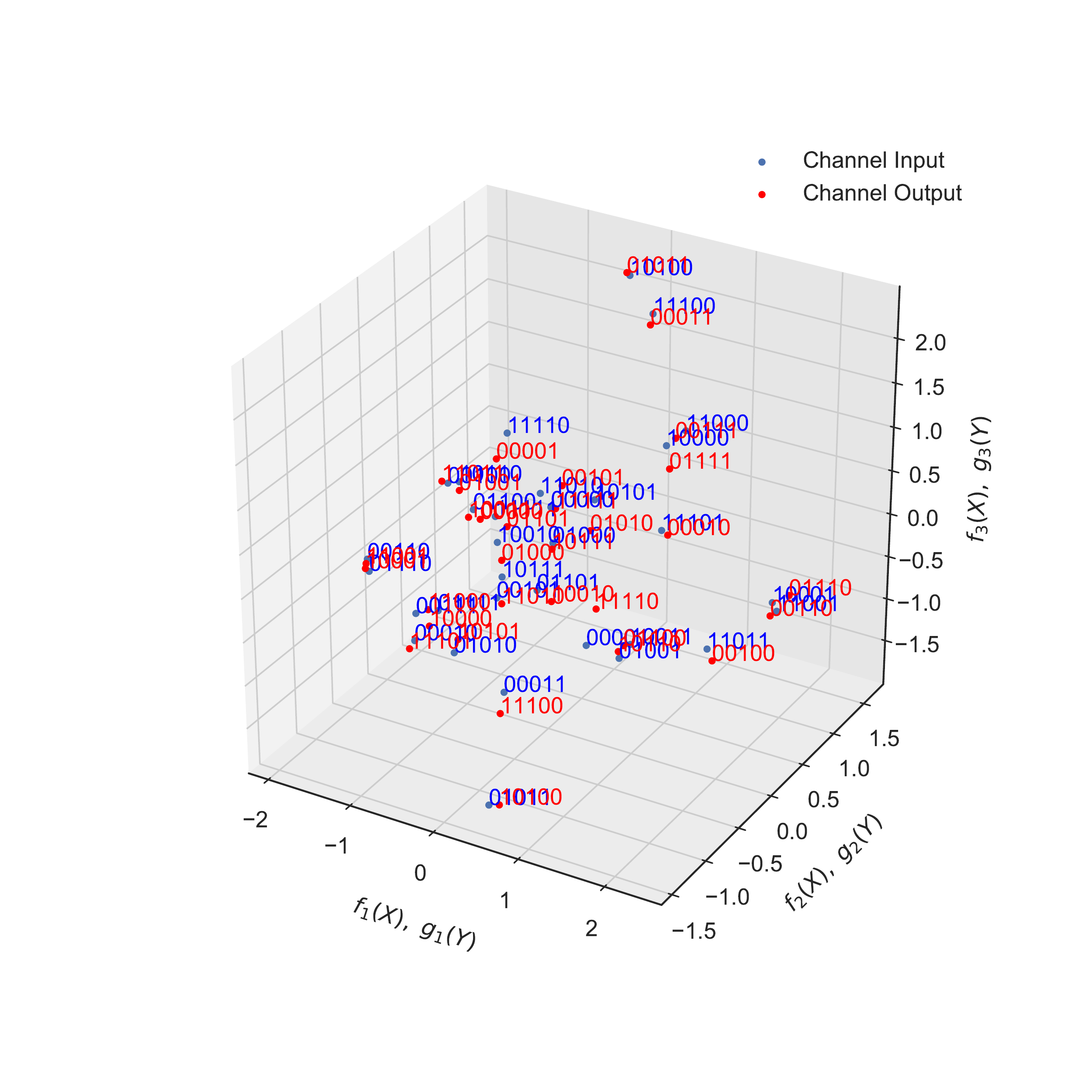}
    }
    \caption{Three-dimensional factoring planes for the BSC with uniform inputs with different crossover probability $\delta$.}
    \label{fig:bsc_delta}
\end{figure}

\begin{figure}[!tb]
    \centering
    \includegraphics[width=.9\textwidth]{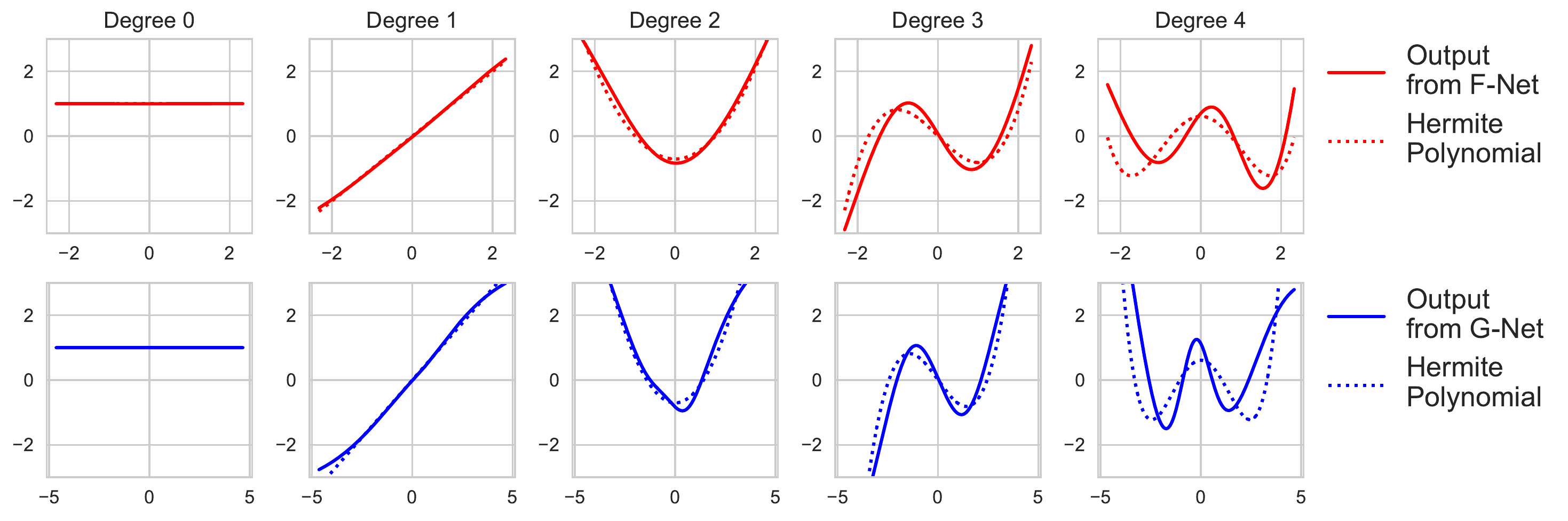}
    \caption{Hermite polynomials of degree $0$ to $4$ and outputs of the CA-NN that approximate the $0^\text{th}$ to $4^\text{th}$ principal functions.}
    \label{fig:hermite_4}
\end{figure}

\begin{table}[!tb]
  \caption{The MSE when using the FG-Net to approximate the principal functions (Hermite polynomials)}
  \label{tab:hermite}
  \centering
  \begin{tabular}{lllllllllll}
     &          & $1^\text{st}$ & $2^\text{nd}$ & $3^\text{rd}$ & $4^\text{th}$
    \\ \hline \\
     & $\mathsf{MSE}_f$ & $0.0001$ & $0.0042$ & $0.0213$ & $0.0522$ \\
     & $\mathsf{MSE}_g$ & $0.0053$ & $0.0197$ & $0.0238$ & $0.0583$ \\        
     & True PICs & $0.6977$ & $0.4675$ & $0.2979$ & $0.2113$ \\
     & CorrA-NN   & $0.7007$ & $0.4938$ & $0.3376$ & $0.2037$ \\
  \end{tabular}
\end{table}

\subsection{Noisy MNIST Dataset}\label{sec:noisy_mnist}
The noisy MNIST dataset \citep{wang2015deep} consists of $28 \times 28$ grayscale handwritten digits, with $60$K/$10$K images for training/testing. Each image is rotated at angles uniformly sampled from $[-\pi/4, \pi/4]$, and random noise uniformly sampled from $[0, 1]$ is added. 
We let $X$ be those images and $Y$ be the ture labels.

The CA-NN is composed of two neural nets with different structures.
Since the inputs of the encoder F-Net are images, we use two convolutional layers with output sizes $32$ and $64$ with filter dimension $5 \times 5$ and max pooling, a fully-connected layer with $1,024$ units, and a readout layer with output size $10$. 
For the G-Net, the inputs are the one-hot encoded labels, and we use two hidden layer with output size $128$ and $64$, respecively, and a readout layer with output size $10$.
We adopt ReLU activation for all hidden layers in the CA-NN.

We train for $200$ epochs on the training set with a batch size of $2048$ using a gradient descent optimizer with a learning rate of $0.01$.
To avoid numerical instability, we clip the outputs of the F-Net to the interval $[-10000, 10000]$. Moreover, when back-propagating the objective in (2), we compute $\bC_{fg}^\intercal (\bC_f^{-1}+\epsilon \bI_d)\bC_{fg}$ instead of $\bC_f^{-1/2}\bC_{fg}$, where $\epsilon = 0.001$ to avoid an invalid matrix inverse. 
Using the reconstitution formula~(3), we reconstruct the likelihood $\pygx$ for classification, and  obtain an accuracy of $99.76\%$ on the training set, $96.77\%$ on the test set.

The PICs are reported in Table~\ref{tab:nmnist_pic}, and the factoring planes drawn with the nine principal functions extracted from training and test set are shown in Figure~\ref{fig:nmnist_train_ca} and Figure~\ref{fig:nmnist_test_ca} respectively.

\subsection{CIFAR-10 Images}
The CIFAR-10 dataset contains $32 \times 32$ colored images, each with three channels representing the RGB color model, along with a label representing one of $10$ categories.
We let $X$ be the images and $Y$ be the labels.
In this experiment, the CA-NN is composed of two neural nets with different structures.
For the F-Net, we use five convolutional layers with max pooling, two fully-connected layers, and a readout layer. 
The convolutional layers have output size $128$, and the filter dimension is $3 \times 3$; the two fully-connected layers have output sizes $384$ and $192$. 
The G-Net has the same architecture as the one we use for training over the noisy MNIST, see the previous Section~\ref{sec:noisy_mnist}. 
We train for $200$ epochs with a batch size of $256$ using a gradient descent optimizer with learning rate $0.001$. 
The accuracy, once again obtained via classification using the likelihood given by the reconstitution formula in (3), is $93.41\%$ on the training set and $89.75\%$ on the test set.
The PICs are reported in Table~\ref{tab:cifar}, and the factoring planes of the nine principal functions extracted from training and test set are shown in Figure~\ref{fig:cifar_train} and Figure~\ref{fig:cifar_test} respectively, where again each colored point corresponds to an image ($X$) differentiated by color for each class, and the black point corresponds to the labels ($Y$).

\subsection{Kaggle What's Cooking Recipe Data}
We first describe how we pre-processed this dataset. Originally the Kaggle What's Cooking Recipe data contains a list of detailed ingredients for each recipe, along with the type of cuisine the dish corresponds to.
We parse the descriptions using Natural Language Toolkit (NLTK) in Python \citep{bird2004nltk} to tokenize the descriptions into a vector of ingredients for each recipe. Next, we keep only the top 146 most common ingredients and discard the others. This is done for visualization purposes on the factorial planes. The output of this process for an example recipe is shown in Table~\ref{tab:whats_cooking}. 

\begin{table}
    \centering
    \begin{tabular}{|c | l |} \hline
        Before & \begin{minipage}{6cm} romaine lettuce, black olives, grape tomatoes, garlic, pepper, purple onion, seasoning, garbanzo beans, feta cheese crumbles\end{minipage} \\[2em]
        \hline
        After  & \begin{minipage}{6cm} onion, garlic, pepper, tomato, lettuce, bean \end{minipage} \\[1em] \hline
    \end{tabular}
    \caption{Effect of the pre-processing and removal of ingredients on a greek recipe.}
    \label{tab:whats_cooking}
\end{table}

The CA-NN is composed of two simple neural nets with $3$ hidden layers, with $30$ units per hidden layers.
Both neural nets adopt hyperpolic tangent activation functions. 
We train the whole dataset for $20000$ epochs by gradient descent optimizer with learning rate $0.005$.
In addition to the first factoring plane shown in the main text, we illustrate the following two factoring planes in Figure~\ref{fig:kaggle_ca_2} and Figure~\ref{fig:kaggle_ca_3} respectively.
Since the PICs of this dataset are large in general, the second and third factoring planes also contain some amount of information. In particular the third principal function allows to separate Indian cuisine from Asian and Western cuisine. Moroccan cuisine is between Indian and Western cuisine on this axis. The fourth principal function separates Asian cuisines into, on one hand Vietnamese and Thai cuisine, and on the other Chinese, Korean and Japanese cuisine. Note that, in this case, there are no signature ingredient, instead it is the entire recipe which helps determining which family of Asian cuisine a dish belongs to.

\subsection{UCI Wine Quality Data}
The CA-NN is composed of two neural nets with different structures.
For the F-Net, we use a simple neural nets with $3$ hidden layers, where the numbers of units at each layer are $500$, $100$, and $30$. 
For the G-Net, we use a simple neural nets with $3$ hidden layers, where the numbers of units at each layer are $10$, $5$, and $3$. 
Both neural nets adopt hyperbolic tangent activation functions. 
We train the whole dataset for $1000$ epochs using an Adam optimizer \citep{kingma2014adam} with learning rate $0.001$.

The PICs are reported in Table~\ref{tab:uci}, and  we illustrate the first two and following two factoring planes in Figure~\ref{fig:red_wine_ca_features} and Figure~\ref{fig:red_wine_ca_34_features} respectively. 
Moreover, we plot the minimum and maximum values of the $11$ features.
In Figure~\ref{fig:red_wine_ca_features}, since we have an additional second factoring plane, we observe that the interpolation path of a low quality and high quality wines does not actually pass through the cluster of medium quality wines.
Since there are only two significant PICs in Table~\ref{tab:uci}, we can see that the third and fourth factoring planes in Figure~\ref{fig:red_wine_ca_34_features} contain barely any information.

\subsection{Influence of the Encoder Net Depths}
We investigate the influence of different configurations of the encoders F and G Nets on the estimation of the PICs. 
Specifically, we adopt the experiment setting in Section~4.1.1, and vary neural network configurations including depth and number of neurons. 
In Table~\ref{tab:pic_config}, we summarize the estimation of the principal inertia components and different configurations of the encoders F and G Nets.
As we can see deeper encoders are prone to overfit the PICs, while shorter and wider encoders are likely to give more accurate estimations of the PICs.

\begin{table}[t!]
\caption{{\small Estimating the PICs with different configurations of the CA-NN.}}
\label{tab:pic_config}
\begin{center}
\begin{tabular}{lclclclclclclclclcl}
\multicolumn{5}{c}{\bf Discrete PICs} 
\\ \hline \\
 & $1^\text{st}$ PIC & $2^\text{nd}$ PIC & $3^\text{rd}$ PIC & $4^\text{th}$ PIC\\
Analytic value & $0.8000$ & $0.8000$ & $0.8000$ & $0.8000$  \\
$30$-$30$-$25$     & $0.8011$ & $0.7942$ & $0.7918$ & $0.7883$  \\
$30$-$30$-$30$-$25$     & $0.8272$ & $0.8217$ & $0.8144$ & $0.7926$  \\
$20$-$20$-$15$     & $0.8259$ & $0.8201$ & $0.8195$ & $0.8075$  \\
$40$-$30$-$20$-$15$     & $0.8363$ & $0.8274$ & $0.8182$ & $0.8020$  \\
$50$-$50$-$30$     & $0.8260$ & $0.8199$ & $0.8193$ & $0.8001$  \\
$60$-$50$-$40$-$30$-$20$     & $0.8226$ & $0.8179$ & $0.8079$ & $0.7972$  
\end{tabular}
\end{center}
\end{table}

\clearpage
\section{Algorithms}\label{app:algo}
\begin{algorithm}[!tb] 
\caption{Recovering $\bF_n(\bx_n)$ and $\bG_n(\by_n)$ from $\tilde{\bF}_n(\bx_n)$ and $\tilde{\bG}_n(\by_n)$, the output of the FG-Nets.}\label{algo:whitening}
\begin{algorithmic}[1] 
\Input $\tilde{\bF}_n(\bx_n)$ and $\tilde{\bG}_n(\by_n)$
\Output Principal functions $\bF_n(\bx_n)$ and $\bG_n(\by_n)$
\State $\tilde{\bF}_n(\bx_n) \gets \tilde{\bF}_n(\bx_n) - \EE{\tilde{\bF}_n(\bx_n)}$,
\Statex $\tilde{\bG}_n(\by_n) \gets \tilde{\bG}_n(\by_n) - \EE{\tilde{\bG}_n(\by_n)}$ \Comment{(Remove mean)}
\State $\bU_f, S_f, \bV_f \gets$ SVD of $\frac{1}{n}\tilde{\bF}_n(\bx_n) \tilde{\bF}_n(\bx_n)^\intercal$,
\Statex $\bU_g, S_g, \bV_g\gets$ SVD of $\frac{1}{n}\tilde{\bG}_n(\by_n) \tilde{\bG}_n(\by_n)^\intercal$
\State $\bC_f^{-1/2} \gets \bU_f S_f^{-1/2} \bV_f^\intercal$,
\Statex $\bC_g^{-1/2} \gets \bU_g S_g^{-1/2} \bV_g^\intercal$ \Comment{(Find inverse)}
\State $\bL = \frac{1}{n} (\bC_f^{-1/2}\tilde{\bF}_n(\bx_n)) (\bC_g^{-1/2}\tilde{\bG}_n(\bx_n))^\intercal$ 
\State $\bU, S, \bV \gets$ SVD of $\bL$ \Comment{(Find singular vectors)}
\State $\bA = \bU^\intercal\bC_f^{-1/2} $, $\bB = \bV^\intercal\bC_g^{-1/2}$
\State \Return $\bA\tilde{\bF}_n(\bx_n)$, $\bB\tilde{\bG}_n(\by_n)$
\end{algorithmic}
\end{algorithm}
Algorithm~\ref{algo:whitening} summarizes how to convert the outputs $\tilde{\bF}_n(\bx_n)$ and $\tilde{\bG}_n(\by_n)$ of the CA-NN to the principal functions by the whitening processing.

\clearpage
\begin{table*}[!t]
  \caption{The PICs of training and test sets for noisy MNIST.}
  \label{tab:nmnist_pic}
  \centering
  \begin{tabular}{lllllllllll}
    PICs  & $1^\text{st}$ & $2^\text{nd}$ & $3^\text{rd}$ & $4^\text{th}$ & $5^\text{th}$ & $6^\text{th}$ & $7^\text{th}$ & $8^\text{th}$ & $9^\text{th}$
    \\ \hline \\
    Training & $0.989$ & $0.987$ & $0.987$ & $0.985$ & $0.982$ & $0.981$ & $0.979$ & $0.978$ & $0.976$   \\
    Test     & $0.957$ & $0.945$ & $0.944$ & $0.927$ & $0.925$ & $0.924$ & $0.921$ & $0.917$ & $0.903$   \\
  \end{tabular}
\end{table*}

\begin{figure*}[!b]
\centering
\includegraphics[width=\textwidth]{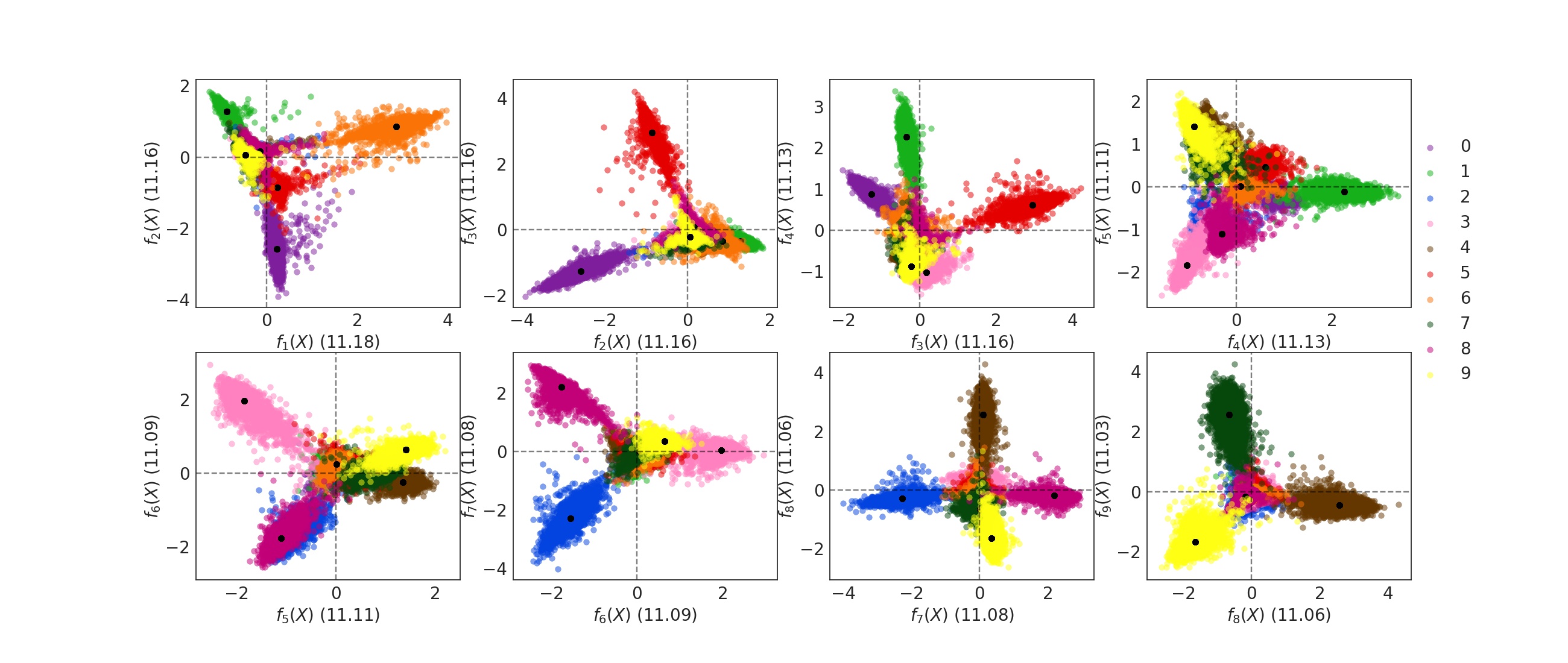}
\caption{Factoring planes of noisy MNIST on training set.}
\label{fig:nmnist_train_ca}
\vspace*{\floatsep}
\includegraphics[width=\textwidth]{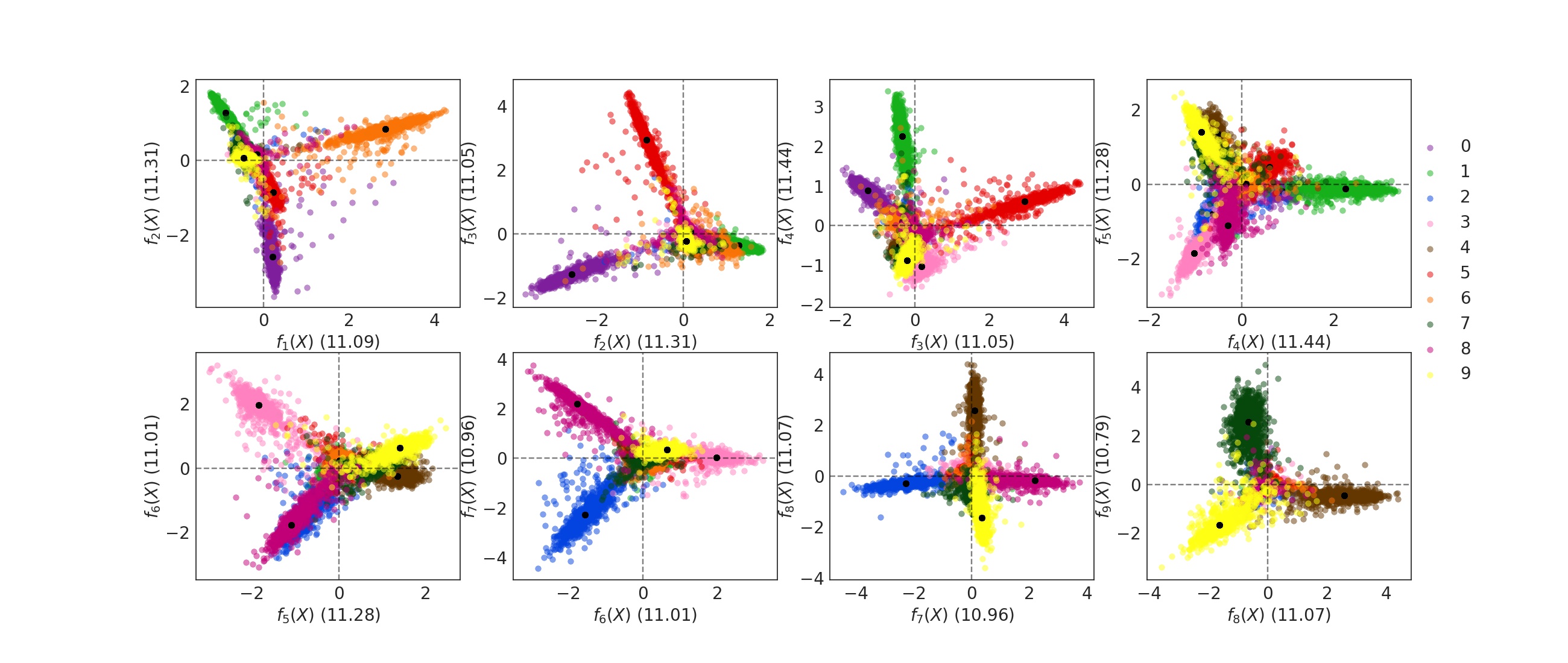}
\caption{Factoring planes of noisy MNIST on test set.}
\label{fig:nmnist_test_ca}
\end{figure*}

\clearpage
\begin{table*}[!t]
  \caption{The PICs of training and test sets for CIFAR-10.}
  \label{tab:cifar}
  \centering
  \begin{tabular}{lllllllllll}
    PICs  & $1^\text{st}$ & $2^\text{nd}$ & $3^\text{rd}$ & $4^\text{th}$ & $5^\text{th}$ & $6^\text{th}$ & $7^\text{th}$ & $8^\text{th}$ & $9^\text{th}$
    \\ \hline \\
    Training & $0.996$ & $0.996$ & $0.996$ & $0.995$ & $0.995$ & $0.994$ & $0.994$ & $0.994$ & $0.993$   \\
    Test    & $0.837$ & $0.800$ & $0.752$ & $0.746$ & $0.739$ & $0.722$ & $0.584$ & $0.562$ & $0.487$   \\
  \end{tabular}
\end{table*}

\begin{figure*}[!b]
\centering
\includegraphics[width=\textwidth]{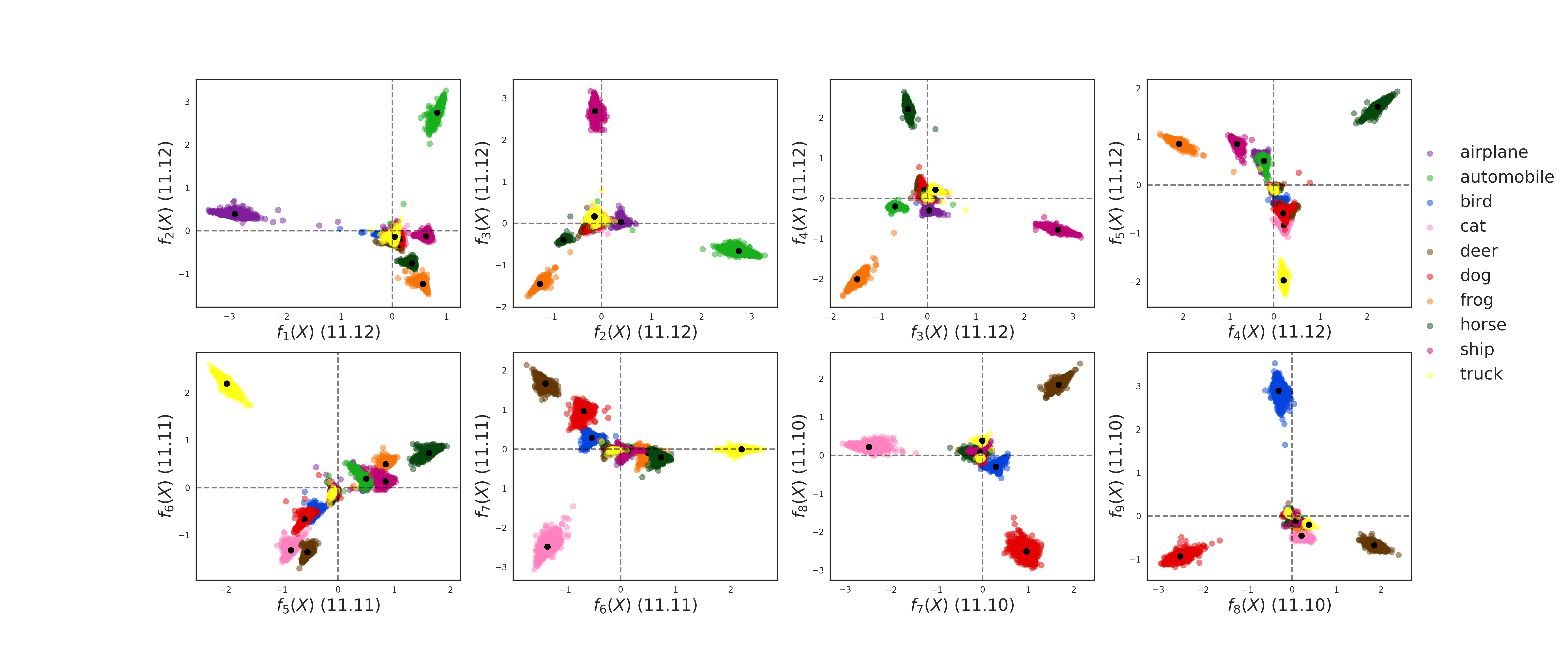}
\caption{Factoring planes of CIFAR-10 on training set.}
\label{fig:cifar_train}
\vspace*{\floatsep}
\includegraphics[width=\textwidth]{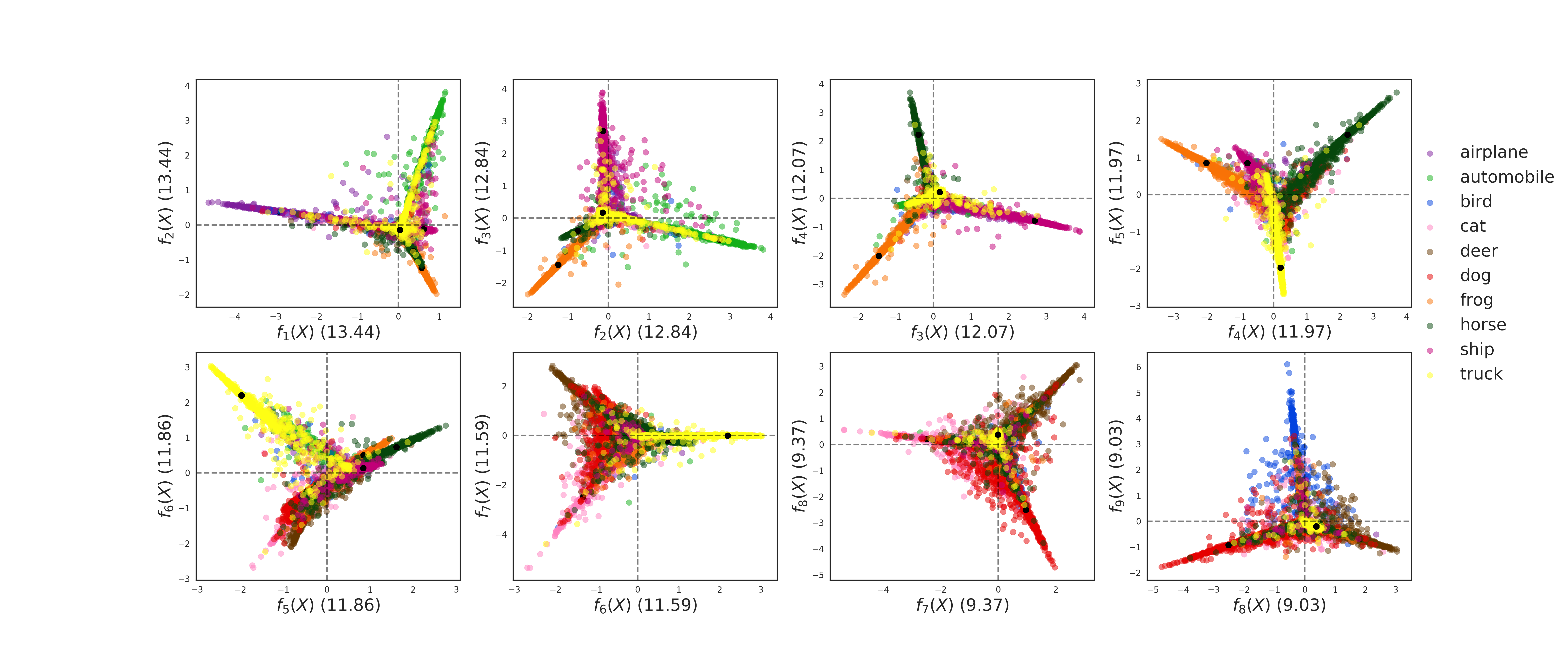}
\caption{Factoring planes of CIFAR-10 on test set.}
\label{fig:cifar_test}
\end{figure*}

\clearpage
\begin{figure*}[!tb]
\centering
\includegraphics[width=\textwidth]{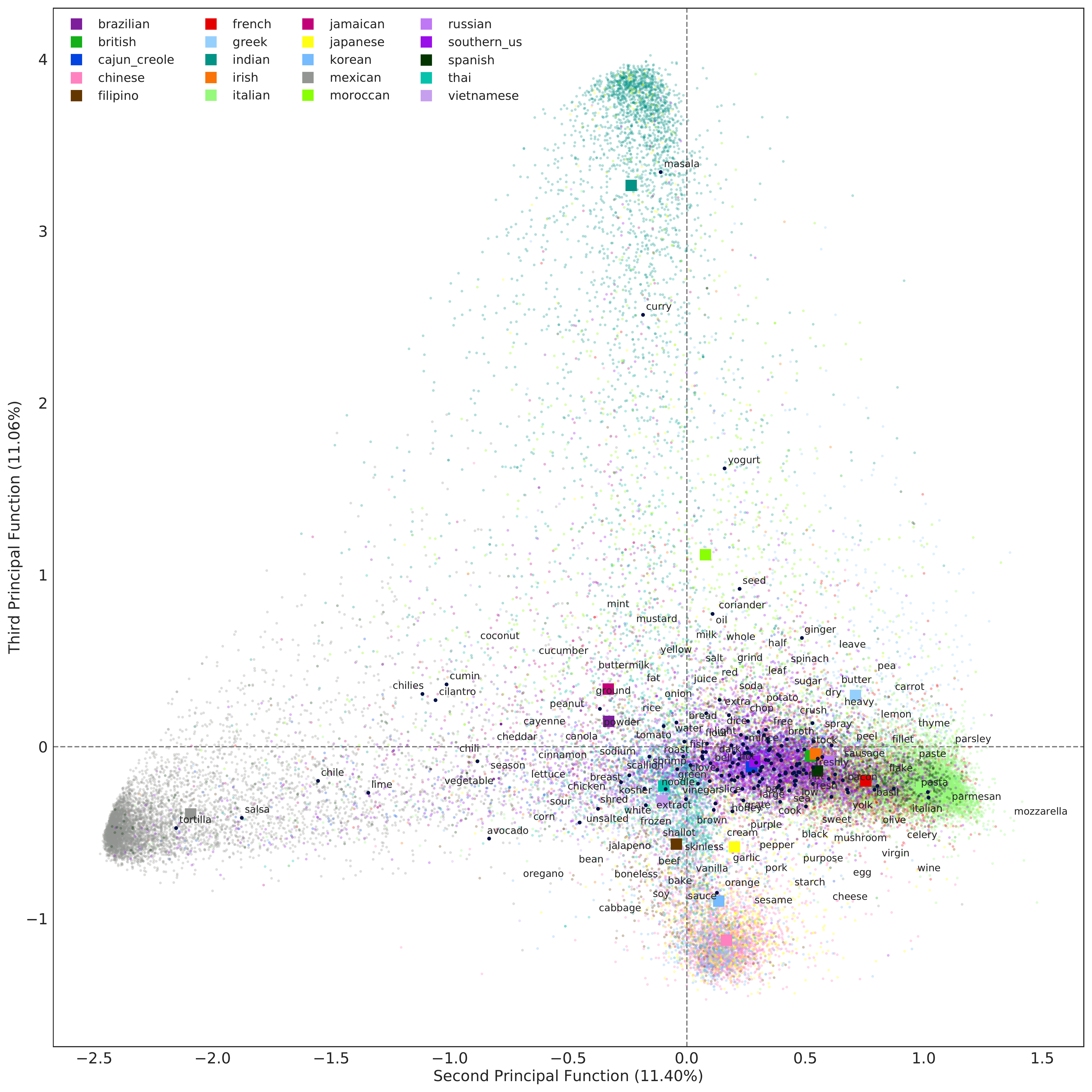}
\caption{The second factoring plane of CA on Kaggle What's cooking dataset (Colored dots: recipe, dark blue: ingredient).}
\label{fig:kaggle_ca_2}
\end{figure*}

\clearpage
\begin{figure*}[!tb]
\centering
\includegraphics[width=\textwidth]{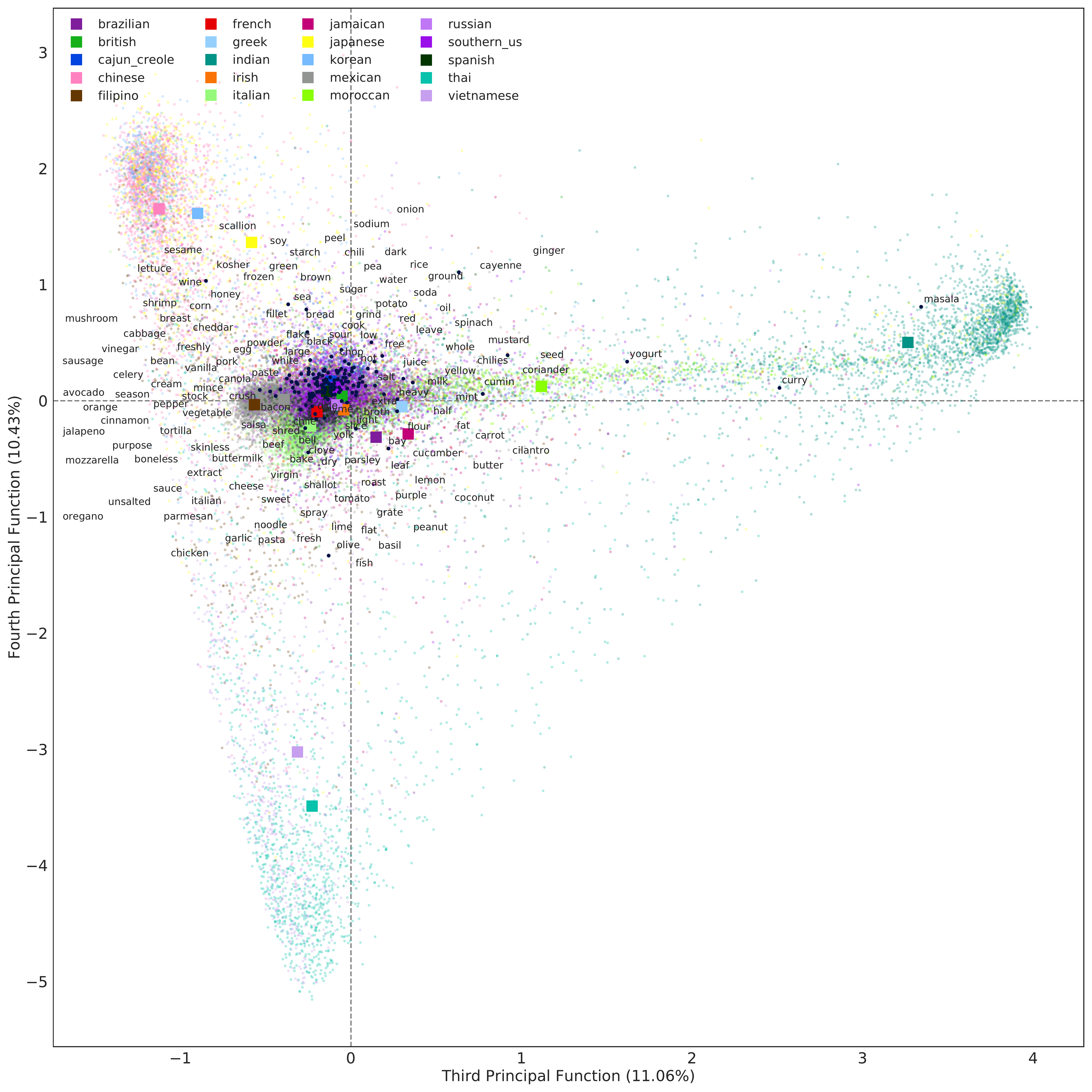}
\caption{The third factoring plane of CA on Kaggle What's cooking dataset (Colored dots: recipe, dark blue: ingredient).}
\label{fig:kaggle_ca_3}
\end{figure*}

\clearpage
\begin{table*}[!t]
  \caption{The PICs of training and test sets for UCI Wine Quality Data.}
  \label{tab:uci}
  \centering
  \begin{tabular}{lllllll}
    PICs  & $1^\text{st}$ & $2^\text{nd}$ & $3^\text{rd}$ & $4^\text{th}$ & $5^\text{th}$ & $6^\text{th}$ 
    \\ \hline \\
    Training & $9.9815e-01$ & $9.9353e-01$ & $5.6861e-02$ & $2.6282e-04$ & $2.0870e-06$ & $1.9238e-27$   \\
    Test     & $9.9984e-01$ & $6.1934e-01$ & $8.8158e-02$ & $2.8603e-04$ & $7.7783e-08$ & $1.4357e-15$   \\
  \end{tabular}
\end{table*}
\begin{figure*}[!tb]
\centering
\includegraphics[width=\textwidth]{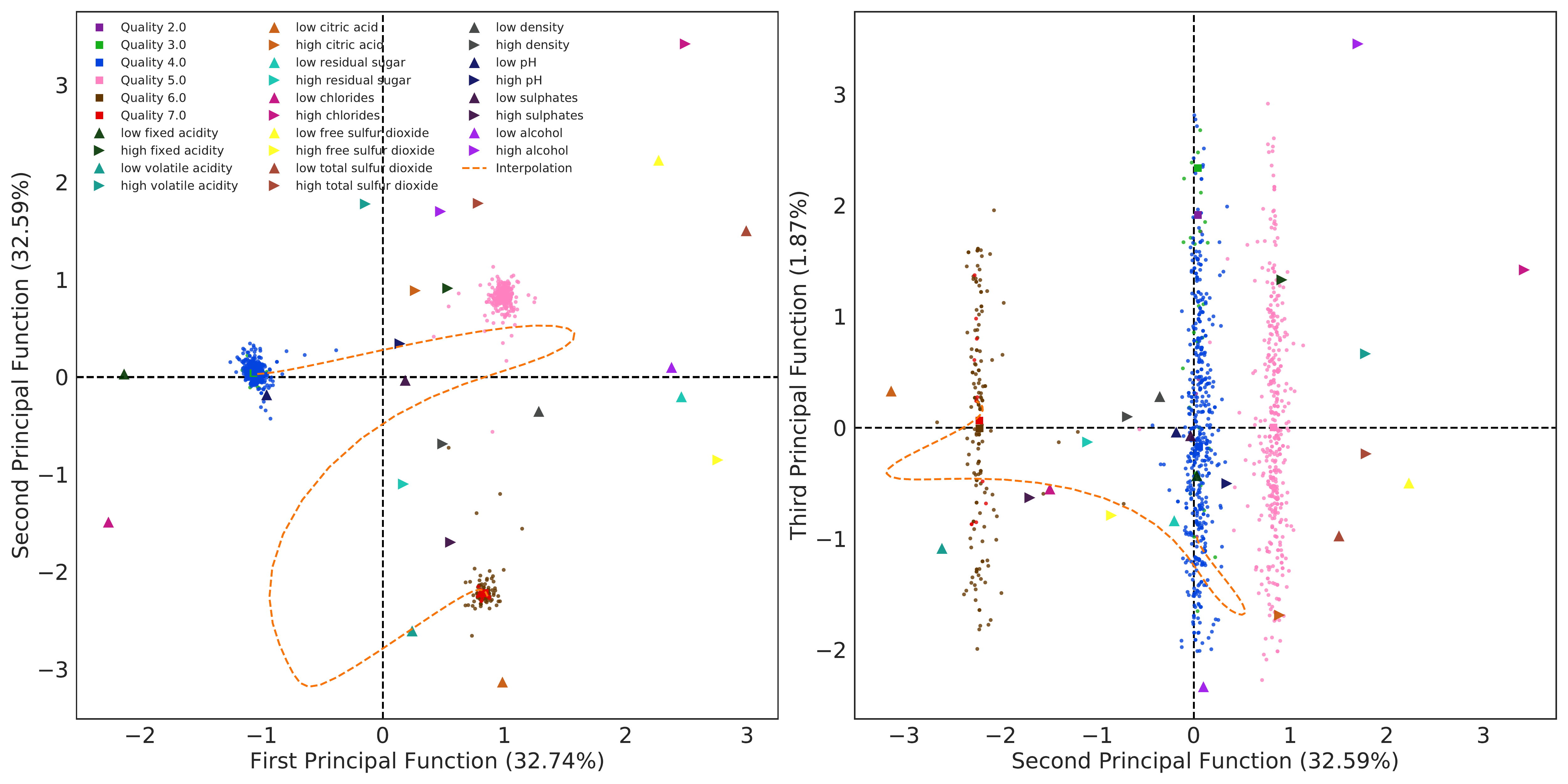}
\caption{The first (left) and second (right) factoring plane of CA on UCI wine quality dataset.}
\label{fig:red_wine_ca_features}
\end{figure*}

\begin{figure*}[!tb]
\centering
\includegraphics[width=\textwidth]{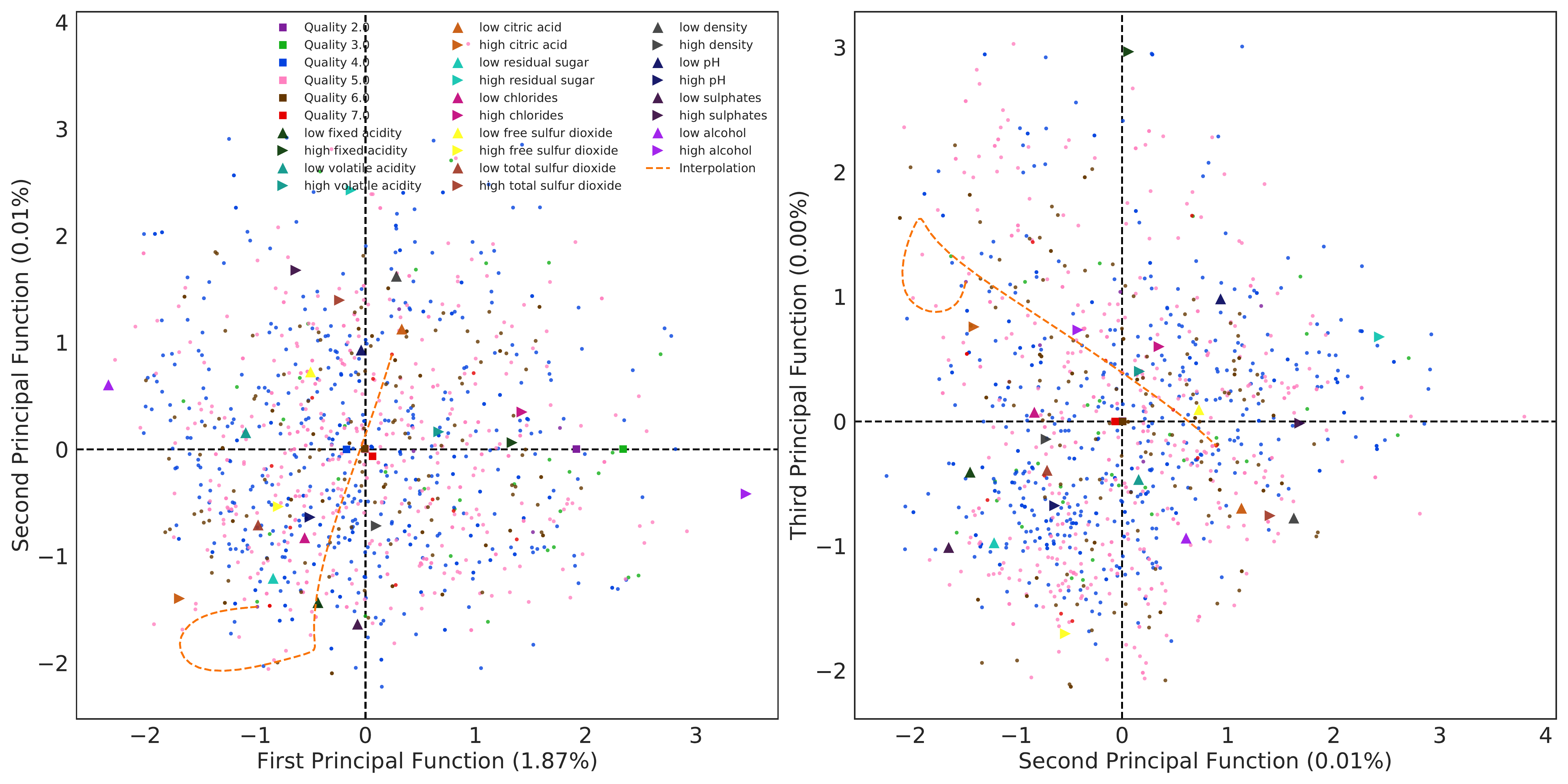}
\caption{The third (left) and fourth (right) factoring plane of CA on UCI wine quality dataset.}
\label{fig:red_wine_ca_34_features}
\end{figure*}

\clearpage
\begin{figure*}[!tb]
    \centering
    \includegraphics[width=\textwidth]{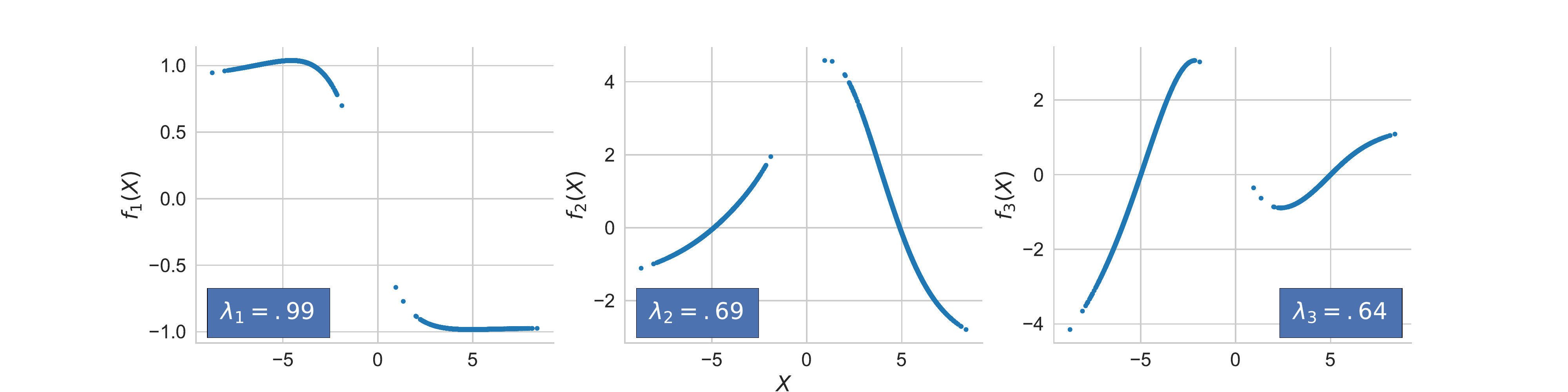}
    \caption{First three principal functions of a multimodal Gaussian, along with the associated PIC values.}
    \label{fig:multi_mode_gaussian}
\end{figure*}
\section{Additional Experiment - Multi-Modal Gaussian}\label{app:add_exp}
As a final set of experiments on synthetic data, we consider mixtures of Gaussian (or multi-modal Gaussian) random variables. More precisely, for $\mu_i \in \mathbb{R}^2, i = 0, 1$, we let $(X,Y) = \mathbf{1}(B = 0) \mathcal{N}(\mu_0, \Sigma) + \mathbf{1}(B = 1) \mathcal{N}(\mu_1, \Sigma)$, where $B \sim \text{Ber}(p)$, and $\mathcal{N}(\mu_i, \Sigma)$ are 2-dimensional multivariate Gaussian random variables with mean $\mu_i$ and covariance matrix $\Sigma$ independent of $B$. In this experiment, we demonstrate the power of the PICs as a fine representation of the relationship between $X$ and $Y$. In particular, letting $\Sigma$ have diagonal elements $1$ and off-diagonal elements $.7$, and letting $\mu_i = (-1)^i [5, 5]^T$, we obtain two modes, one at $[-5, -5]$ and the other at $[5,5]$. First, note that a general measure of dependence such as Mutual information, would be unable to capture the existence of two modes. In fact, one can verify that the mean-zero jointly Gaussian pair $(\tilde{X},\tilde{Y})$ which has correlation $.93$ satisfy $I(X;Y) = I(\tilde{X},\tilde{Y}) \approx 1.03$ nats. Despite this, the relationship between $X$ and $Y$ is different from the relationship between $\tilde{X}$ and $\tilde{Y}$, as exhibited by the principal functions Fig.~\ref{fig:multi_mode_gaussian}. Specifically, note that the first principal function distinguishes between the two modes. The second and third principal functions capture the two dimensional space of piece wise linear-function, where each mode follows a separate linear function. When it comes to the value of the PICs, we see that the top PIC is very close to 1, while  the top PIC of $(\tilde{X},\tilde{Y})$ is given by the correlation, i.e. $.93$. However, when it comes to estimating linear functions, one can perform better inference over $(\tilde{X}, \tilde{Y})$, since the PIC for this family of function is of about $.7$ in the multi-modal gaussian.

\section{Proofs}
\subsection{Proposition~\ref{prop:CA_PIC}}\label{app:proof_1}
If we write (\ref{eq:P_YgX}) into matrix form and following the definitions in Section~\ref{sec:implementation}, we have
\begin{eqnarray}\label{eq:pic_to_ca}
\mathbf{F} \mathbf{\Lambda}  \mathbf{G}^\intercal &=& \mathbf{D}_{X}^{-1}\mathbf{P}_{X,Y}\mathbf{D}_{Y}^{-1}-\mathbf{1}_{|\calX|}\mathbf{1}_{|\calY|}^\intercal\\
&=& \mathbf{D}_{X}^{-1}(\mathbf{P}_{X,Y}-\mathbf{p}_X\mathbf{p}_Y^\intercal)\mathbf{D}_{Y}^{-1}\\
&=& \mathbf{D}_{X}^{-1/2}\mathbf{Q}\mathbf{D}_{Y}^{-1/2}\\
&=& \mathbf{D}_{X}^{-1/2}\bU \bSigma \bV^\intercal\mathbf{D}_{Y}^{-1/2}\\
&=& \mathbf{L}\bSigma \mathbf{R}^\intercal,
\end{eqnarray}
where $[\mathbf{F}]_{i, j} = f_j(i)$, $[\mathbf{G}]_{i, j} = g_j(i)$ and $\mathbf{\Lambda} = \textsf{diag}(\lambda_0, \cdots, \lambda_d)$.
Eq.~(\ref{eq:pic_to_ca}) shows that in discrete case, the principal functions $\mathbf{F}$ and $\mathbf{G}$ are equivalent to the orthogonal factors $\mathbf{L}$ and $\mathbf{R}$ in the CA, and the factoring scores $\bSigma$ are the same as the PICs $\mathbf{\Lambda}$. The reconstitution formula in (\ref{eq:P_YgX}) actually connects the PICs and correspondence analysis, and enables us to generalize correspondence analysis to continuous variables \citep{hirschfeld1935connection, gebelein1941statistische}.

\subsection{Proposition~\ref{prop:opti}}\label{app:proof_2}
Since the objective (\ref{opti2}) can be expressed as
\begin{eqnarray}\label{sm:opti2}
&&\mathbb{E}[\|\bA\mathbf{\tilde{f}}(X)-\mathbf{\tilde{g}}(Y)\|^2_2] = \text{tr} \left( \bA\mathbb{E}[ \mathbf{\tilde{f}}(X)\mathbf{\tilde{f}}(X)^\intercal ]\bA^\intercal \right) \nonumber \\
&& - 2\text{tr} \left( \bA\mathbb{E}[ \mathbf{\tilde{f}}(X)\mathbf{\tilde{g}}(Y)^\intercal] \right) + \left( \mathbb{E}[\|\mathbf{\tilde{g}}(Y)\|^2_2] \right),
\end{eqnarray}
we have 
\begin{equation}
    \mathbb{E}[\|\bA\mathbf{\tilde{f}}(X)-\mathbf{\tilde{g}}(Y)\|^2_2] = d - 2\text{tr} \left( \bA \bC_{fg} \right) + \mathbb{E}[\|\mathbf{\tilde{g}}(Y)\|^2_2],
\end{equation}
where the last equation comes from the fact that $\text{tr} \left( \bA\mathbb{E}[ \mathbf{\tilde{f}}(X)\mathbf{\tilde{f}}(X)^\intercal ]\bA^\intercal \right) = \text{tr} \left( \mathbf{I}_d \right) = d$. Since $\bC_f$ is positive-definite, $C_f^{-\frac{1}{2}}$ exists, and so does $\bA = \Tilde{\bA}\bC_f^{-\frac{1}{2}}$, and (\ref{sm:opti2}) can be alternatively expressed as
\begin{eqnarray}
\begin{aligned}
\min\limits_{\bA \in \Reals^{d\times d},\mathbf{\tilde{f}},\mathbf{\tilde{g}}} &\; -2\text{tr}(\Tilde{\bA}\bB) + \mathbb{E}[\|\mathbf{\tilde{g}}(Y)\|^2_2]\\
\text{subject to}&\; \Tilde{\bA}\Tilde{\bA}^\intercal = \mathbf{I}_d,
\end{aligned}
\end{eqnarray}
where $\bB = \bC_f^{-\frac{1}{2}}\bC_{fg}$.
The term $\text{tr}(\Tilde{\bA}\bB)$ can be upper bounded by the Von Neumann's trace inequality \citep{mirsky1975trace},
\begin{equation}
    \text{tr}(\Tilde{\bA}\bB) \leq \sum_{i=1}^d \sigma_{\Tilde{\bA}, i}\sigma_{\bB, i},
\end{equation}
where $\sigma_{\Tilde{\bA}, i}$'s and $\sigma_{\bB, i}$'s are the singular values for $\tilde{\bA}$ and $\bB$ respectively. 
Moreover, the upper bounded can be achieved by solving the orthogonal Procrustes problem \citep{gower2004procrustes}, and the optimizer is $\Tilde{\bA}^* = \bV\bU^\intercal$, where $\bV$ and $\bU$ are given by the SVD of $\bB = \bU\mathbf{\Sigma}_\bB \bV^\intercal$.
Therefore, 
\begin{equation}
    \text{tr}(\Tilde{\bA}^*\bB) = \text{tr}(\bV\bU^\intercal \bU\mathbf{\Sigma}_\bB \bV^\intercal) = \sum_{i=1}^d \sigma_{\bB, i}
\end{equation}
which is the $d$-th Ky-Fan norm of $\bB$. The desired result then follows by simple substitution.

\end{document}